\documentclass[conference,10pt,twocolumn]{IEEEtran}

\usepackage[utf8]{inputenc}
\usepackage{latexsym}
\usepackage{amsmath}
\usepackage{amsthm}
\usepackage{amssymb}
\usepackage{mathtools}
\usepackage{graphicx}
\usepackage{ifthen}
\usepackage{calc}
\usepackage{stmaryrd}
\usepackage{enumitem}
\usepackage{alphalph}

\usepackage[backend=bibtex,style=numeric-comp,url=false,giveninits=true,maxnames=10]{biblatex}
\usepackage[dvipsnames]{xcolor}
\usepackage[colorlinks=true,citecolor=ForestGreen,linkcolor=RubineRed,urlcolor=RoyalBlue]{hyperref}

\usepackage{tikz}
\usetikzlibrary{arrows}
\usetikzlibrary{shapes.geometric}
\usetikzlibrary{graphs}
\usetikzlibrary{calc}
\usetikzlibrary{3d}
\tikzset{vertex/.style={circle,draw,fill=RoyalBlue!60,minimum size= 5pt,inner sep = 0mm}}

\usepackage{eucal}

\DeclareFontFamily{OT1}{pzc}{}
\DeclareFontShape{OT1}{pzc}{m}{it}{<-> s * [1.10] pzcmi7t}{}
\DeclareMathAlphabet{\mathpzc}{OT1}{pzc}{m}{it}

\newtheoremstyle{mythm}{}{}{\itshape}{}{\bfseries}{.}{.5em}{\thmname{#1}~\thmnumber{#2}\ifthenelse{\equal{\thmnote{#3}}{}}{}{~(\thmnote{#3})}}

\newtheoremstyle{mydefn}{}{}{\upshape}{}{\bfseries}{.}{.5em}{\thmname{#1}~\thmnumber{#2}\ifthenelse{\equal{\thmnote{#3}}{}}{}{~(\thmnote{#3})}}

\newtheoremstyle{myremark}{}{}{\upshape}{}{\itshape}{.}{.5em}{\thmname{#1}~\thmnumber{#2}\ifthenelse{\equal{\thmnote{#3}}{}}{}{~(\thmnote{#3})}}

\theoremstyle{mythm}
\newtheorem{theorem}{Theorem}[section]
\newtheorem{lemma}[theorem]{Lemma}
\newtheorem{proposition}[theorem]{Proposition}
\newtheorem{corollary}[theorem]{Corollary}

\theoremstyle{mydefn}

\newtheorem{example}[theorem]{Example}
\theoremstyle{myremark}
\newtheorem{remark}[theorem]{Remark}
\theoremstyle{mythm}

\newcommand{\uend}{\hfill$\lrcorner$}
\newcommand{\uende}{\eqno\lrcorner}

\newcounter{qcounter}
\newenvironment{question}{\refstepcounter{qcounter}\medskip\par\noindent\textbf{Question~\arabic{qcounter}.}\itshape}{\par\medskip}

\newcounter{claimcounter}

\setenumerate[1]{label=(\arabic*)}
\newlist{eroman}{enumerate}{2}
\setlist[eroman,1]{label=(\roman*)}
\setlist[eroman,2]{label=(\alph*)}
\newlist{ealph}{enumerate}{1}
\setlist[ealph]{label=(\Alph*)}

\renewcommand{\theequation}{\Alph{equation}}

\usepackage{color}
\definecolor{blau}{RGB}{0,84,159}
\definecolor{hellblau}{RGB}{142,168,229}
\definecolor{petrol}{RGB}{0,97,101}
\definecolor{tuerkis}{RGB}{0,152,161}
\definecolor{gruen}{RGB}{87,171,39}
\definecolor{maigruen}{RGB}{189,205,0}
\definecolor{gelb}{RGB}{255,237,0}
\definecolor{orange}{RGB}{255,128,0}
\definecolor{magenta}{RGB}{227,0,102}
\definecolor{rot}{RGB}{204,7,30}
\definecolor{bordeaux}{RGB}{161,16,53}
\definecolor{violett}{RGB}{97,33,88}
\definecolor{lila}{RGB}{122,111,172}
\definecolor{grey}{gray}{0.7}
\definecolor{mittelblau}{RGB}{0,128,255}

\definecolor{azure}{RGB}{6,154,243}
\definecolor{xindigo}{RGB}{56,2,130}
\definecolor{xorange}{RGB}{255,173,1}
\definecolor{xgreen}{RGB}{21,176,26}

\newcommand{\bigmid}{\mathrel{\big|}}
\newcommand{\Bigmid}{\mathrel{\Big|}}

\renewcommand{\hat}{\widehat}

\renewcommand{\vec}[1]{\boldsymbol{#1}}

\newcommand{\lmulti}{{\{\hspace{-3.4pt}\{}}
\newcommand{\rmulti}{{\}\hspace{-3.4pt}\}}}
\newcommand{\biglmulti}{{\big\{\hspace{-4.3pt}\big\{}}
\newcommand{\bigrmulti}{{\big\}\hspace{-4.3pt}\big\}}}
\newcommand{\Biglmulti}{{\Big\{\hspace{-5.1pt}\Big\{}}
\newcommand{\Bigrmulti}{{\Big\}\hspace{-5.1pt}\Big\}}}

\renewcommand{\phi}{\varphi}
\renewcommand{\epsilon}{\varepsilon}

\newcommand{\Nat}{{\mathbb N}}
\newcommand{\Real}{{\mathbb R}}

\newcommand{\logic}[1]{\textsf{\upshape #1}}
\newcommand{\LL}{{\logic L}}
\newcommand{\LC}{{\logic C}}
\newcommand{\LCk}[1]{{\logic C}_{#1}}
\newcommand{\LCq}[1]{{\logic C}^{(#1)}}
\newcommand{\LCkq}[2]{{\logic C}_{#1}^{(#2)}}
\newcommand{\LGC}{{\logic{GC}}}
\newcommand{\LGCk}[1]{{\logic{GC}}_{#1}}
\newcommand{\LGCkq}[2]{{\logic{GC}}_{#1}^{(#2)}}

\newcommand{\FO}{\logic{FO}}
\newcommand{\TC}{\logic{TC}}

\newcommand{\CC}{{\mathcal C}}

\newcommand{\CF}{{\mathcal F}}

\newcommand{\CN}{{\mathcal N}}

\newcommand{\CQ}{{\mathcal Q}}

\DeclareMathOperator{\sig}{sig}
\DeclareMathOperator{\lsig}{lsig}
\DeclareMathOperator{\relu}{relu}

\newcommand{\col}{\logic{col}}
\newcommand{\atp}[1]{\logic{atp}_{#1}}
\newcommand{\colref}[1]{\logic{cr}^{(#1)}}
\newcommand{\wl}[2]{\logic{wl}_{#1}^{(#2)}}
\newcommand{\owl}[2]{\logic{owl}_{#1}^{(#2)}}
\newcommand{\agg}{\logic{agg}}
\newcommand{\comb}{\logic{comb}}

\newcommand{\ro}{\logic{ro}}
\newcommand{\aggro}{\logic{aggro}}

\begin{document}
\title{The Logic of Graph Neural Networks}
\author{\IEEEauthorblockN{Martin Grohe}
\IEEEauthorblockA{RWTH Aachen University, Germany\\
  Email: \href{mailto:grohe@informatik.rwth-aachen.de}{grohe@informatik.rwth-aachen.de}}
}
\maketitle

\begin{abstract}
  Graph neural networks (GNNs) are deep learning architectures for 
  machine learning problems on graphs. It has recently been
  shown that the expressiveness of GNNs can be characterised precisely
  by the combinatorial Weisfeiler-Leman algorithms and by finite
  variable counting logics. The correspondence has even led to new, higher-order GNNs corresponding to the
  WL algorithm in higher dimensions.

  The purpose of this paper is to explain these descriptive
  characterisations of GNNs.
\end{abstract}

\thispagestyle{plain}

\section{Introduction}
\label{sec:intro}
Graph neural networks (GNNs) are deep learning architectures for
graph structured data that have developed into a method of choice for
many graph learning problems in recent years. It is, therefore,
important that we understand their power. One aspect of this is the
expressiveness: which functions on graphs can be expressed by a GNN
model? This may guide us in our choice for a suitable learning
architecture for a learning problem we want to solve, and perhaps
point to a suitable variant of GNNs.

Machine learning (statistical reasoning) and logic (symbolic
reasoning) not always seem to be compatible. Therefore, it comes as a
pleasant surprise that the expressive power of GNNs
has a clean logical characterisations in terms of finite variable
counting logics. These logics have been studied for a long time in
finite model theory, and they play an important role in the field
(see, e.g., \cite{ebbflu99,imm99,ott97,lib04}). This is because of their
role in analysing fixed-point logic with counting and the quest for a
logic capturing polynomial time \cite{ott97,gro17}, and because of
their close relation to the Weisfeiler-Leman (WL) algorithm
\cite{caifurimm92,weilem68}: the $k$-dimensional WL
algorithm is an equivalence test for the $(k+1)$-variable counting logic
$\LCk{k+1}$.

The $k$-dimensional WL algorithm iteratively computes a colouring of
the $k$-tuples of vertices of a graph by passing local information
about the isomorphism type and the current colour between tuples. The
goal is to detect differences; tuples that are structurally different
(formally: belong to different orbits of the automorphism group)
should eventually receive different colours. It has been shown by Cai,
Fürer and Immerman \cite{caifurimm92} that this goal cannot always be
reached, but still, the WL algorithm gathers ``most'' structural
information. For graph classes excluding some fixed graph as a minor,
we can always find a $k$ such that $k$-dimensional WL fully
characterises the graphs in this class and the orbits of their
automorphism groups \cite{gro17}, for planar graphs even $k=3$ is
enough \cite{kieponschwe17} (see \cite{kie20} for a recent
survey). For most of this paper, we will focus on the $1$-dimensional
WL algorithm and a variant known as colour refinement.\footnote{Often,
  no distinction is made between 1-dimensional WL and colour
  refinement, because in some sense, they are equivalent (see
  Proposition~\ref{prop:cr1wl}). However, it will be important here to
  emphasise the difference between the algorithms.} On the logical
side, they correspond to the $2$-variable counting logic $\LCk2$ and
its guarded fragment $\LGCk2$, which is also known as graded modal
logic.

Intuitively, the similarity between GNNs and the 1-dimensional WL
algorithm is quite obvious. A
GNN represents a message passing algorithm operating on  graphs; just like the
WL algorithm, a GNN works by
iteratively passing local information along the edges of a graph. At
each point during the computation, each
vertex gets a real-valued vector as its state. Vertices exchange
information by sending messages along the edges of the graph, and then
they update their states based on their current state and the messages
they receive. The message functions as well as the state update
functions are represented by neural networks, and their parameters can
be learned, either from labelled examples to learn a model of an
unknown function defined on the vertices of a graph, or in an
unsupervised fashion to learn a vector representation of the vertices
of a graph. Crucially, all vertices use the
same message passing and state update functions, that is, the
parameters of the neural network are shared across the graph. This idea
is inspired by convolutional neural networks, which have been applied
extremely successfully in computer vision. It not
only reduces the overall number of parameters that need to be learned,
but it also guarantees that the learned function is isomorphism
invariant (or equivariant), which is essential for learning functions
on graphs.

It has been proved independently by Morris et al.\ \cite{morritfey+19}
and Xu et al.\ \cite{xuhulesjeg19} that the colour refinement
algorithm precisely captures the expressiveness of GNNs in the sense
that there is a GNN distinguishing two nodes of a graph if and only if
colour refinement assigns different colours to these nodes. This
implies a characterisation of the distinguishability of nodes by GNNs
in terms of equivalence in the logic $\LGCk2$. Barcelo et
al.~\cite{DBLP:conf/iclr/BarceloKM0RS20} extended this result and
showed that every property of nodes definable by a $\LGCk2$-formula
is expressible by a GNN, uniformly over all graphs. Corresponding
results hold for GNNs with an additional feature that allows vertices
to access aggregate global information,
the 1-dimensional WL algorithm, and the logic $\LCk2$.  It is the main
purpose of this paper to explain these results, including all required
background, all the fine-print in their statements, and their proofs.

The results have some interesting extensions. Based on the
higher-order WL algorithm, Morris et al.\ \cite{morritfey+19}
introduced higher-order GNNs that lift the equivalence between WL and
GNNs to a higher level. Another interesting extension is to let GNNs
operate with a random initialisation of their states. This is
often done in practice anyway. At first sight, this violates
the isomorphism invariance, but as random variables, the functions
computed by GNNs with random initialisation remain invariant. It turns out
that random initialisation significantly affects the
expressiveness; all invariant functions on graphs can be expressed
using GNNs with random
initialisation~\cite{abbceygroluk21}.

The paper is organised as follows. After giving the necessary
preliminaries, we include background sections on finite-variable
counting logics, invariant and equivariant functions defined on
graphs, the Weisfeiler-Leman algorithm, and feedforward neural
networks. We then introduce GNNs. In the
subsequent sections, we state and prove the results outlined above as well
as their consequences. We conclude with a brief discussion and a
number of open questions.

\section{Preliminaries}
\label{sec:prel}
We denote the set of real numbers by $\Real$ and the set of
nonnegative integers by $\Nat$. For real numbers $x\le y$, by $(x,y)$ and $[x,y]$ we
denote the open resp.\ closed interval between $x$ and $y$. For every positive integer $n$, we let
$[n]=\{1,\ldots,n\}$.

We use bold-face letters to denote tuples (or
vectors). The entries of a $k$-tuple $\vec x$ are $x_1,\ldots,x_k$;
the length $k$ will usually be clear from the context. For a $k$-tuple
$\vec x$ and an element $y$, by $\vec xy$ we denote
the $(k+1)$-tuple $(x_1,\ldots,x_k,y)$. Moreover, for every $i\in[k]$,
by $\vec x[y/i]$ we denote the $k$-tuple
$(x_1,\ldots,x_{i-1},y,x_{i+1},\ldots,x_k)$ and by $\vec x[/i]$ the
$(k-1)$-tuple $(x_1,\ldots,x_{i-1},x_{i+1},\ldots,x_k)$.

A \emph{multiset} is an unordered collection with repetition, which
can formally be described as a function from a set to the positive
integers. We use $\lmulti\ldots\rmulti$ to denote multisets. We write
$M\subseteq X$ to denote that $M$ is a multiset with elements from a
set $X$ and $x\in M$ to denote that $x$ appears at least once in $M$.
The \emph{order} of a multiset is the sum of the multiplicities of all
elements of $M$.

By default, we assume \emph{graphs} to be finite, undirected, simple,
and vertex-labelled. Thus formally, a graph is a tuple
$G=(V(G),E(G),P_1(G),\ldots,P_\ell(G))$ consisting of a finite vertex
set $V(G)$, a binary edge relation $E(G)\subseteq V(G)^2$ that is
symmetric and irreflexive, and unary relations
$P_1(G),\ldots,P_\ell(G)\subseteq V(G)$ representing $\ell$ vertex
labels.  We usually denote edges without parenthesis, as in $vw$, with
the understanding that $vw=wv$. We always use $\ell$ to denote the
number of labels of a graph; if $\ell=0$, we speak of an
\emph{unlabelled graph}. Throughout this paper, we think of $\ell$ as
being fixed and not part of the input of computational problems.  We
sometimes refer to the vector
$\col(G,v)\coloneqq(c_1,\ldots,c_\ell)\in\{0,1\}^\ell$ with $c_i=1$ if
$v\in P_i(G)$ and $c_i=0$ otherwise as the \emph{colour} of vertex
$v\in V(G)$. Let me remark that $\col(G,v)$ is even defined for
unlabelled graphs, where we simply have $\col(G,v)=()$ (the empty
tuple) for all $v$.

An isomorphism from a graph $G$ to a graph $G'$ is a bijective mapping
$f:V(G)\to V(G')$ that preserves edges as well as labels, that is,
$vw\in E(G)\iff f(v)f(w)\in E(G')$ for
all $v,w\in V(G)$ and $v\in P_i(G)\iff f(v)\in P_i(G')$ for all
$i\in[\ell]$ and $v\in V(G)$.

We use the usual graph theoretic terminology. In particular, the set
of \emph{neighbours} of a vertex $v$ in a graph $G$ is
$N^G(v)\coloneqq\{w\in V(G)\mid vw\in E(G)\}$.
For a set $X\subseteq V(G)$, the \emph{induced subgraph} $G[X]$ is the
graph with vertex set $X$, edge relation $E(G)\cap X^2$, and unary
relations $P_i(G)\cap X$. The \emph{order} of a graph $G$ is
$|G|\coloneqq|V(G)|$.

\begin{remark}\label{rem:bin}
  To keep the presentation simple, in this paper we focus on
  undirected vertex-labelled graphs, but all the results can be
  extended to directed graphs with edge labels, that is, binary
  relational structures. Formally, a binary relational structure is a
  tuple  $A=(V(A),E_1(A),\ldots,E_k(A),P_1(A),\ldots,P_\ell(A))$
  consisting of a finite vertex set $V(A)$, binary relations
  $E_i(A)\subseteq V(A)$ and unary relations $P_j(A)$. For such structures, we define
  the colour map $\col(A):V(A)\to\{0,1\}^{k+\ell}$ by
  $\col(A,v)\coloneqq(c_1,\ldots,c_k,d_1,\ldots,d_\ell)$ with $c_i=1$
  if $(v,v)\in E_i(A)$ and $d_j=1$ if $v\in P_j(A)$.

  We will get back to binary relational structures once in a while,
  mainly in a series of remarks.
  \uend
\end{remark}

\begin{remark}
  Another common generalisation is to assign real-valued weights to
  vertices and possibly edges instead of the Boolean labels. Formally,
  instead of the subsets $P_i(G)\subseteq V(G)$ we have functions
  $P_i(G):V(G)\to\Real$. Isomorphisms $f$ are required to preserve
  these functions. Again, most results can be extended to this
  setting, but will not discuss these extensions here.  \uend
\end{remark}

\section{Finite Variable Counting Logic}
\label{sec:logic}
The logics we are interested in are fragments of the
extension $\LC$ of first-order logic $\FO$ by \emph{counting
  quantifiers} $\exists^{\ge p}$. That is, $\LC$-formulas are formed
from \emph{atomic formulas} of the form $x=y$, $E(x,y)$, $P_i(x)$ with the
usual Boolean connectives, and the new counting quantifiers. Standard
existential and universal quantifiers can easily be expressed using
the counting quantifiers ($\exists x$ as $\exists^{\ge 1}x$ and
$\forall x$ as $\neg\exists^{\ge 1}x\neg$).

We interpret $\LC$-formulas over labelled graphs; variables range
over the vertices. We use the notation
$\phi(x_1,\ldots,x_k)$ to indicate that the free variables of a
formula $\phi$ are among $x_1,\ldots,x_k$, and for a graph $G$ and
vertices $v_1,\ldots,v_k\in V(G)$ we write $G\models\phi(v_1,\ldots,v_k)$
to denote that $G$ satisfies $\phi$ if the variables $x_i$ are
interpreted by the vertices $v_i$. The semantics
of the counting quantifiers is the obvious one: a labelled graph $G$
together with vertices $w_1,\ldots,w_k\in V(G)$ satisfies a formula $\exists^{\ge p}x\,\phi(x,y_1,\ldots,y_k)$ if there are
at least $p$ vertices $v\in V(G)$ such that
$G\models\phi(v,w_1,\ldots,w_k)$.

\begin{remark}\label{rem:bin-logic}
  We can easily extend $\LC$ from graphs to arbitrary (binary)
  relational structures; we only need to add additional atomic
  formulas $E_i(x,y)$ for the binary relations $E_i$.
  \uend
\end{remark}

Note that while syntactically an extension of $\FO$, the logic $\LC$ has
the same expressive power as $\FO$, because the formula $\exists^{\ge
  p}x\,\phi(x)$ is equivalent to 
\[
\exists x_1\ldots\exists
x_p\Big(\bigwedge_{1\le i<j\le p}x_i\neq
x_{j}\wedge\bigwedge_{i=1}^p\phi(x_i)\Big).
\]
However, the translation from $\LC$ to $\FO$ incurs an increase in the
number of variables as well as the quantifier rank (maximum number of
nested quantifiers). For example, the $\LC$-formula $\forall
x\exists^{\ge d}y E(x,y)$ stating that the minimum degree of a graph
is $d$ uses 2 variables and has quantifier rank $2$. It is easy to
show that any equivalent $\FO$-formula needs at least $d+1$ variables
and has quantifier rank at least $d+1$.

By $\LCk k$ we
denote the fragment of $\LC$ consisting of all formulas with at most
$k$ variables, and by $\LCq q$, we denote the fragment consisting of all
formulas of quantifier rank at most $q$.\footnote{This notation is
  slightly non-standard, but it aligns well with our notation for the
  WL algorithm and GNNs later.} We also combine the two,
letting $\LCkq kq\coloneqq\LCk k\cap\LCq q$. The fragments are still
quite expressive.

\begin{example}
  For every $\ell$ one can easily
construct a $\LCkq3\ell$-formula stating that the diameter of a graph
is at most $2^\ell$ based on the inductive definition of distances:
\[\delta_{2n}(x,y)=\exists
  z\big(\delta_n(x,z)\wedge\delta_n(z,y)\big).\uende\]
\end{example}

We need to consider one more fragment of the logic $\LC$, the
\emph{guarded fragment} $\LGC$. In this fragment, quantifiers are
restricted to range over the neighbours of the current
nodes. Formally, $\LGC$-formulas are formed from the atomic formulas
by the Boolean connectives and quantification restricted to formulas
of the form $\exists^{\ge p}y(E(x,y)\wedge\psi)$, where $x$ and $y$
are distinct variables and $y$ appears in
$\psi$ as a free variable. Note
that every formula of $\LGC$ has at least one free variable.
We are mainly interested in the
$2$-variable fragment $\LGCk2$, also known as \emph{graded
modal logic}
\cite{DBLP:conf/iclr/BarceloKM0RS20,DBLP:journals/sLogica/Rijke00}. Again
we use a superscript to indicate the quantifier rank, that is,
$\LGCkq2q\coloneqq\LGCk2\cap\LCq q$. 

\begin{example}
  The following $\LGCk2$-formula $\phi(x)$ says that vertex $x$ has at most 1
  neighbour that has more than 10 neighbours with label $P_1$:
  \[
    \phi(x)\coloneqq \neg \exists^{\ge 2}y\Big(E(x,y)\wedge\exists^{\ge
      11}x\big(E(y,x)\wedge P_1(x)\big)\Big).\uende
  \]
\end{example}

\section{Invariants and Colourings}
In this paper, we study mappings defined on graphs or vertices of
graphs. Crucially, we want these mappings to be isomorphism
invariant. A \emph{graph invariant} (or \emph{$0$-ary graph invariant}) is a
function $\xi$ defined on graphs such that $\xi(G)=\xi(G')$ for
isomorphic graphs $G,G'$.
For $k\ge1$, a \emph{$k$-ary graph
  invariant} is a function $\xi$ that associates with
each graph $G$ a function $\xi(G)$ defined on $V(G)^k$ in such a way
that for all graphs $G,G'$, all isomorphisms $f$ from $G$ to $G'$, and
all tuples $\vec v\in V(G)^k$ it holds that
\[
  \xi(G)(\vec v)=\xi(f(G))(f(\vec v)).
\]
Formally, such a mapping $\xi$ is called \emph{equivariant}.
In
the following, we just speak of \emph{$k$-ary
  invariants}. Moreover, we call $0$-ary invariants \emph{graph
  invariants} and $1$-ary invariants \emph{vertex invariants}. To simplify the notation, for mappings
$\xi$ (equivariant or not) that associate a function on $V(G)^k$ with
every graph $G$, we usually write $\xi(G,\vec v)$ instead of
$\xi(G)(\vec v)$. We have already used this notation for the vertex
invariant $\col$ defined in Section~\ref{sec:prel}.

Let $\xi$ be a $k$-ary invariant. If $\xi(G,\vec v)\neq\xi(G',\vec
v')$, then we say that $\xi$ \emph{distinguishes} $(G,\vec v)$ and
$(G',\vec v')$ (or just $\vec v$ and $\vec v'$ if the graphs are clear
from the context). The equivariance condition implies that if  $\xi$ distinguishes $(G,\vec v)$ and
$(G',\vec v')$ then there is no isomorphism from $G$ to $G'$ that maps
$\vec v$ to $\vec v'$. If the converse also holds, then $\xi$ is a
\emph{complete} invariant.

From a $k$-ary invariant $\xi$ we can derive a
graph invariant $\hat\xi$ by mapping each graph $G$ to the multiset
\[
  \hat\xi(G)\coloneqq \biglmulti \xi(G,\vec
  v)\bigmid \vec v\in V(G)^k\bigrmulti.
\]
We say that $\xi$ \emph{distinguishes} two graphs $G,G'$ if
$\hat\xi(G)\neq\hat\xi(G')$. Note that if $\xi$ is a complete
invariant then $\hat\xi$ is complete as well, that is, $\xi$
distinguishes $G,G'$ if and only if $G,G'$ are non-isomorphic. The
converse does not necessarily hold, that is, there are incomplete
invariants $\xi$ for which $\hat\xi$ is complete.

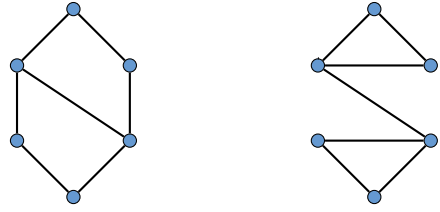
\begin{figure}
  \centering
  \begin{tikzpicture}
    \begin{scope}
    \node[vertex] (v1) at (0.75,0) {}; 
    \node[vertex] (v2) at (0,0.75) {}; 
    \node[vertex] (v3) at (0,1.75) {}; 
    \node[vertex] (v4) at (0.75,2.5) {}; 
    \node[vertex] (v5) at (1.5,1.75) {}; 
    \node[vertex] (v6) at (1.5,0.75) {};
    \draw[thick] (v1) edge (v2) edge (v6) (v3) edge (v2) edge (v4)
    edge (v6) (v5) edge (v4) edge (v6);
    \end{scope}

      \begin{scope}[xshift=4cm]
    \node[vertex] (v1) at (0.75,0) {}; 
    \node[vertex] (v2) at (0,0.75) {}; 
    \node[vertex] (v3) at (0,1.75) {}; 
    \node[vertex] (v4) at (0.75,2.5) {}; 
    \node[vertex] (v5) at (1.5,1.75) {}; 
    \node[vertex] (v6) at (1.5,0.75) {};
    \draw[thick] (v1) edge (v2) edge (v6) (v2) edge (v6) (v3) edge (v3) edge (v4)
    edge (v5) edge (v6) (v5) edge (v4);
    \end{scope}
\end{tikzpicture}
  \caption{Two graphs with the same degree sequence}
  \label{fig:deg}
\end{figure}

\begin{example}\label{exa:invs}
  \begin{enumerate}
  \item 
  The degree $\logic{deg}$ defined by
  $\logic{deg}(G,v)\coloneqq |N^G(v)|$ is a vertex invariant. The graph invariant
  $\widehat{\logic{deg}}$ associates with each graph the multiset of
  vertex degrees appearing in the graph (or, equivalently, the
  \emph{degree sequence}). Note that $\logic{deg}$ does not
  distinguish the two graphs shown in Figure~\ref{fig:deg}.
\item
  The binary invariant $\logic{dist}$ is defined by letting
  $\logic{dist}(G,v,w)$ be the length of the shortest path from $v$ to
  $w$ in $G$, or $\infty$ if no such path exists. $\logic{dist}$
  distinguishes the two graphs in Figure~\ref{fig:deg}.

  \item We define a ternary invariant $\logic{tri}$ by letting
    $\logic{tri}(G,v_1,v_2,v_3)=1$ if $G$ induces a triangle on
    $\{v_1,v_2,v_3\}$ and $\logic{tri}(G,v_1,v_2,v_3)=0$
    otherwise. Then $\widehat{\logic{tri}}$ is essentially the graph
    invariant ``number of triangles''. $\logic{tri}$
  distinguishes the two graphs in Figure~\ref{fig:deg} as well.
  \uend
  \end{enumerate}
\end{example}

We often think of mappings $\chi$ defined on $V(G)$ or $V(G)^k$ as
\emph{colourings} of the vertices or $k$-tuples of vertices of a graph
$G$ (where ``colour'' is just an intuitive way of illustrating an
abstract range). If the range of $\chi$ is $\Real^n$, we also call
$\chi$ a $n$-dimensional \emph{feature map}, in particular in the
context of graph neural networks.  We refer to the elements in the
range of a colouring (or feature map) $\chi$ as \emph{colours} and to
the pre-images $\chi^{-1}(c)$ for colours $c$ as \emph{colour
  classes}. Thus a colouring defined on $V(G)^k$ induces a partition
of $V(G)^k$ into colour classes.  For colourings
$\chi,\chi':V(G)^k\to C$, we say that $\chi$ \emph{refines} $\chi'$
(we write $\chi\preceq\chi'$) if for all $\vec v,\vec w\in V(G)^k$, if
$\chi(\vec v)=\chi(\vec w)$ then $\chi'(\vec v)=\chi'(\vec w)$. In
other words: the partition of $V(G)^k$ into the colour classes of
$\chi$ refines the partition into the colour classes of $\chi'$. We
call colourings $\chi,\chi'$ \emph{equivalent} (we write
$\chi\equiv\chi'$) if $\chi\preceq\chi'$ and $\chi'\preceq\chi$, that
is, if $\chi$ and $\chi'$ induce the same partition. For $k$-ary
invariants $\xi,\xi'$ we say that $\xi$ \emph{refines} $\xi'$ if
$\xi(G)$ refines $\xi'(G)$ for every graph $G$, and similarly for
equivalence.

The vertex colouring $\col$ is, in some sense, the most basic vertex
invariant. We extend it to a $k$-ary invariant $\atp k$. For a
graph $G$ with $\ell$ labels and a tuple
$\vec v=(v_1,\ldots,v_k)\in V(G)$, we let
$\atp k(G,\vec v)\in\{0,1\}^{2\binom{k}{2}+k\ell}$ be the vector which
for $1\le i<j\le k$ has two entries indicating whether $v_i=v_j$ and
whether $v_iv_j\in E(G)$ and for $1\le i\le k$ has $\ell$ entries
indicating whether $v_i$ is in $P_1(G),\ldots,P_\ell(G)$. Note that
$\atp 1=\col$. The vector $\atp k(G,\vec v)$ is called the
\emph{atomic type} of $\vec v$ in $G$. The crucial property of atomic
types is that $\atp k(G,\vec v)=\atp k(G',\vec v')$ if and only if the
mapping $v_i\mapsto v_i'$ is an isomorphism from the induced subgraph
$G[\{v_1,\ldots,v_k\}]$ to the induced subgraph
$G'[\{v_1',\ldots,v_k'\}]$. In particular, this implies that the
 $\atp k$ is indeed an invariant.

\begin{example}
  Recall the invariants $\logic{dist}$ and $\logic{tri}$ defined in
  Example~\ref{exa:invs}. 
  For all unlabelled graphs $G$ (that is, graphs with $\ell=0$ labels),
  $\logic{dist}(G)$ refines $\atp 2(G)$ and $\atp 3(G)$ refines
  $\logic{tri}(G)$. In general, both refinements are strict.
  \uend
\end{example}

\begin{remark}\label{rem:bin-atp}
  If we want to extend atomic types to arbitrary relational
  structures, say with $m$ binary and $\ell$ unary relations, we
  extend the range to be $\{0,1\}^{mk^2+\binom{k}{2}+\ell k}$. For each binary
  relation, we reserve $n^2$ bits storing
  which pairs it contains.
  \uend
\end{remark}

\section{The Weisfeiler-Leman Algorithm}
The Weisfeiler-Leman (WL) algorithm was originally introduced as an
isomorphism test, and it served that purpose well. It is a key
subroutine of individualisation-refinement algorithms on which all
modern graph isomorphism tools build
\cite{CodenottiKSM13,JunttilaK11,Lopez-PresaCA14,McKay81,McKayP14},
and it has played an important role in theoretical research towards
the graph isomorphism problem \cite{bab16,bab81,caifurimm92,gro17}. It
has long been recognised that the WL algorithm has a close connection
with finite variable counting logics \cite{caifurimm92,immlan90}. Only recently, several
remarkable connections of the WL algorithm to other areas have
surfaced
\cite{abrdawwan17,atsman13,atsmanrob+19,bergro15,delgrorat18,dvo10,kermlagar+14},
and building on them applications ranging from linear optimisation to
machine learning
\cite{grokermla+14,morkermut17,morritfey+19,sheschlee+11,xuhulesjeg19}.

The 1-dimensional version of the WL algorithm, which is essentially
the same as the \emph{colour refinement algorithm}, has been
re-invented several times; the oldest reference I am aware of is
\cite{mor65}.  The $2$-dimensional version, also known as the
\emph{classical} Weisfeiler-Leman algorithm, has been introduced by
Weisfeiler and Leman \cite{weilem68}, and the $k$-dimensional
generalisation goes back to Babai and Mathon (see~\cite{caifurimm92}).

Let us start by introducing the \emph{colour refinement algorithm},
which is also known as \emph{naive vertex classification}. As I said, it is
essentially the same as the 1-dimensional WL algorithm, but there is a
subtle difference that will be relevant for us here. The basic idea is
to label vertices of the graph with their iterated degree sequence.
More precisely, the initial colouring given by the vertex labels in a
graph is repeatedly refined by counting, for each colour, the number of
neighbours of that colour. For every graph $G$, we define a sequence
of vertex colourings $\colref t(G)$ as follows: for every $v\in
V(G)$, we let $\colref 0(G,v)\coloneqq\col(G,v)$ and
\[
  \colref{t+1}(G,v)\coloneqq\Big(\colref t(G,v),\biglmulti\colref t(G,w)\bigmid w\in
  N(v)\bigrmulti\Big).
\]
Thus for vertices $v,v'\in
V(G)$ we have $\colref{t+1}(G,v)=\colref{t+1}(G,v')$ if and only if for every
colour $c$ in the range of $\colref t(G)$ the vertices $v$ and $v'$ have the
same number of neighbours $w$ with $\colref t(G,w)=c$.

In each round, the algorithm computes a colouring that is finer than
the one computed in the previous round, that is,
$\colref{t+1}(G) \preceq \colref t(G)$. For some $t<n\coloneqq|G|$, this
procedure stabilises, meaning the colouring does not become strictly
finer anymore. Hence, there is a least $t_\infty < n$ such that
$\colref{t_\infty+1}(G) \equiv \colref {t_\infty}(G)$.  We call
$\colref{t_\infty}(G)$ the \emph{stable colouring} and denote it by
$\colref{\infty}(G)$.

It is easy to see that for every $t\in\Nat\cup\{\infty\}$, the mapping
$\colref t$ is equivariant; in other words: $\colref t$ is a vertex
invariant.

\begin{example}\label{exa:cr-inxp}
  The stable colouring
  $\colref{\infty}$ does not distinguish the two graphs in
  Figure~\ref{fig:deg}.
  Neither does it distinguish a cycle of length $6$ from a pair of
  triangles.

  Thus the invariant $\colref{\infty}$ is not complete.
  \uend
\end{example}

The colour refinement algorithm is very efficient.
The stable colouring of a graph can be computed in time
$O((n+m)\log n)$, where $n$ denotes the number of vertices and $m$ the
number of edges of the input graph \cite{carcro82} (also
see~\cite{paitar87})\footnote{To be precise, the algorithms computes the partition of the
  vertex set corresponding to the stable colouring, not the actual
  colours viewed as multisets.}.
For a natural class of partitioning algorithms,
this is best-possible~\cite{berbongro17}.

Let us now turn to the \emph{$k$-dimensional Weisfeiler-Leman
  algorithm ($k$-WL)}. It defines a sequence of $k$-ary invariants
$\wl kt$ for $t\in\Nat\cup\{\infty\}$. For every graph $G$ and $\vec
v\in V(G)^k$, we
let $\wl k0(G,\vec v)\coloneqq\atp k(G,\vec v)$ and 
\[
  \wl k{t+1}(G,\vec v)\coloneqq \big(\wl kt(G,\vec v),M\big)
\]
with 
\begin{align*}
  M=
  \Biglmulti
  \Big(\atp{k+1}(G,\vec v w),\,&\wl kt\big(G,\vec v[w/1]\big),\\
                              &\wl kt\big(G,\vec v[w/2]\big),\\
                              &\hspace{1cm}\vdots\\
  &\wl kt\big(G,\vec v[w/k]\big)
               \Big)
  \Bigmid
  w\in V(G)\Bigrmulti.
\end{align*}
Recall the notation $\vec v[w/i]=(v_1,\ldots,v_{i-1},w,v_{i+1},\ldots
v_k)$.

Then there is a least $t_\infty < n^k$
such that $\wl k{t_\infty}(G)\equiv \wl k{t_\infty+1}(G)$, and we call
$\wl k{\infty}(G)\coloneqq \wl k{t_\infty}(G)$ the \emph{$k$-stable colouring}.
The $k$-stable colouring of an $n$-vertex graph can be computed  in
time $O(k^2n^{k+1}\log n)$ \cite{immlan90}.

Recall that an invariant is \emph{complete} if two tuples receive the
same colour if and only if there is an isomorphism mapping one to the
other and that an invariant \emph{distinguishes} two graphs if they
get the same multiset of colours. If an invariant is complete, then it
distinguishes any two non-isomorphic graphs. The following example
shows that neither $1$-WL nor $2$-WL are complete.

\begin{figure}
 \centering
 \begin{tikzpicture}[
   every node/.style={draw,circle,inner sep=0pt,minimum
     height=2mm}
   ]
  \foreach \i in {0,...,3}{
   \foreach \j in {0,...,3}{
    \node[vertex] (v\i\j) at ($(90*\i:0.8) + (90*\i+35*\j-52.5:1.2)$) {};
   }
  }
  
  \foreach \i in {0,...,3}{
   \foreach \j/\k in {0/1,0/2,0/3,1/2,1/3,2/3}
   \path[draw,thick] (v\i\j) edge (v\i\k);
  }
  
  \foreach \i in {0,...,3}{
   \foreach \j/\k in {0/1,0/2,0/3,1/2,1/3,2/3}
   \path[draw,thick] (v\j\i) edge (v\k\i);
  }
  \begin{scope}[xshift=-1.5cm]
  \foreach \i in {0,...,7}{
   \node[vertex] (u\i) at ($(6,0)+(22.5+45*\i:1)$) {};
   \node[vertex] (v\i) at ($(6,0)+(22.5+45*\i:2)$) {};
  }
  \foreach \i/\j in {0/2,0/3,0/5,0/6,1/3,1/4,1/6,1/7,2/4,2/5,2/7,3/5,3/6,4/6,4/7,5/7}{
   \path[draw,thick] (u\i) edge (u\j);
  }
  \foreach \i/\j in {0/1,0/7,1/0,1/2,2/1,2/3,3/2,3/4,4/3,4/5,5/4,5/6,6/5,6/7,7/0,7/6}{
   \path[draw,thick] (v\i) edge (u\j);
  }
  \foreach \i/\j in {0/1,0/2,0/6,0/7,1/2,1/3,1/7,2/3,2/4,3/4,3/5,4/5,4/6,5/6,5/7,6/7}{
   \path[draw,thick] (v\i) edge (v\j);
 }
 \end{scope}
 \end{tikzpicture}
 \caption{Two non-isomorphic strongly regular graphs with parameters $(16,6,2,2)$: the line graph of $K_{4,4}$ (left) and the Shrikhande Graph (right).}
 \label{fig:sr-graphs}
\end{figure}
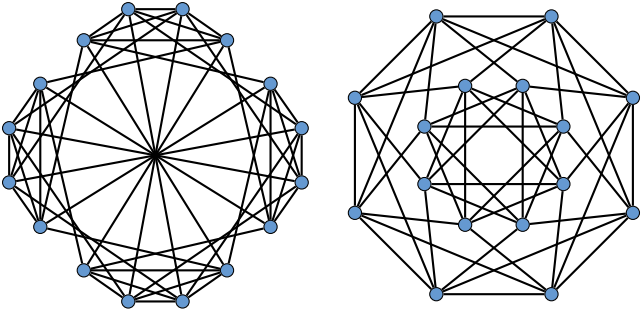

\begin{example}
  \begin{enumerate}
  \item Let $G,G'$ be regular graphs of the same degree, for example a
    cycle of length $6$ and the union of two triangles. Then
    $\wl1{\infty}$ does not distinguish $G$ and $G'$.
  \item Let $G,G'$ be strongly regular graphs with the same
    parameters (see Figure~\ref{fig:sr-graphs}). Then $\wl2{\infty}$ does not distinguish $G$ and $G'$.
  \end{enumerate}
\end{example}

In a seminal paper, Cai, Fürer, and Immerman \cite{caifurimm92}
proved that $\wl k{\infty}$ is incomplete for $k\ge 3$.

\begin{theorem}[\cite{caifurimm92}]
  For every $k\ge 3$ there are non-isomorphic 3-regular graphs of
  order $O(k)$ that $\wl k{\infty}$ does not distinguish.
\end{theorem}

The next proposition and the following remark clarify the relation
between colour refinement and $1$-WL.

\begin{proposition}\label{prop:cr1wl}\label{prop:cr1wl}
  For all $t\in\Nat\cup\{\infty\}$, graphs $G,G'$ are distinguished by $\colref t$ if and
  only if they are distinguished by $\wl1t$.
\end{proposition}

\begin{proof}[Proof sketch]
  Observe that the difference between colour refinement and 1-WL is
  that $\colref{t+1}(G)$ only counts neighbours of a node $v$ in each
  colour class of $\colref t(G)$ to determine $\colref{t+1}(G,v)$, whereas 
$\wl1{t+1}(G)$ counts both neighbours and non-neighbours of $v$ in each
  colour class of $\wl1t(G)$ to determine $\wl1{t+1}(G,v)$. 

  Now the
  idea is that if we have a one-to-one correspondence between the
  colours in the range of $\colref t(G)$ and $\wl1t(G)$ and all
  corresponding colour classes have the same sizes in $G$ and $G'$,
  then it is enough to count neighbours in each colour class because
  the number of non-neighbours in a colour class is the size of the
  class minus the number of neighbours. 

  We can easily turn this intuition into an inductive proof.
\end{proof}

\begin{remark}
  It can be proved similarly to Proposition~\ref{prop:cr1wl} that for all graphs $G$ and all vertices
  $v,v'$ of $G$ we have
  \[
    \colref t(G,v)=\colref t(G,v')
    \iff
    \wl1t(G,v)=\wl1t(G,v').
  \]
  But note that $\colref t(G,v)=\colref t(G',v')$ does not imply
  $\wl1t(G,v)=\wl1t(G',v')$ for distinct graph $G,G'$. As an example,
  let $G,G'$ be cycles of different lengths and $v,v'$ arbitrary
  vertices. Then for all $t\in\Nat$ it holds that $\colref
  t(G,v)=\colref t(G,v')$, but $\wl1t(G,v)\neq\wl1t(G',v')$ for $t\ge
  1$.
  \uend
\end{remark}

There is a variant of the WL algorithm that can also be found in the
literature (for example, \cite{gro00}). We call it \emph{oblivious
  WL}. It is somewhat simpler, and it has the advantage of being
closer to both the logical characterisation of WL and 
higher-order graph neural networks. Its main disadvantage is that it
is less efficient. To reach the same expressive power as $k$-WL,
oblivious WL needs to operate on on $(k+1)$-tuples and thus needs memory
space $\Omega(n^{k+1})$, whereas $k$-WL only needs space $O(n^k)$.

Oblivious $k$-WL defines a sequence of $k$-ary invariants
$\owl kt$ for $t\in\Nat\cup\{\infty\}$. For every graph $G$ and $\vec
v\in V(G)^k$, we
let $\owl k0(G,\vec v)\coloneqq\atp k(G,\vec v)$ and 
\begin{align*}
  &\owl k{t+1}(G,\vec v)\coloneqq\\
  &\hspace{1cm}\bigg(\owl kt(G,\vec v),
\begin{array}[t]{@{\;}l}
\Biglmulti \owl kt\big(G,\vec
v[w/1]\big)\Bigmid w\in V(G)\Bigrmulti,\\
\Biglmulti \owl kt\big(G,\vec
v[w/2]\big)\Bigmid w\in V(G)\Bigrmulti,\\
\hspace{2cm}\vdots\\
\Biglmulti \owl kt\big(G,\vec
v[w/k]\big)\Bigmid w\in V(G)\Bigrmulti\bigg).
\end{array}
\end{align*}
Then we can define the stable colouring $\owl k{\infty}(G)$ in the
same way as we did for colour refinement and standard WL.

The difference between oblivious WL and standard WL is that when we
substitute a vertex $w$ in the $i$th place of $\vec v$ in oblivious
WL, we forget about the context, that is, the entry $v_i$ we
substitute as well as the colours that we obtain if we substitute
other entries of $\vec v$ by $w$. Standard WL considers all these
colours together, whereas oblivious WL is oblivious to the context
(hence the name). It turns out that the price for obliviousness is an
increase of the dimension by $1$.

\begin{theorem}\label{theo:owl}
  Let $k\ge 1$. Then for all
  $t\in\Nat$ and all graphs $G,G'$:
  \begin{enumerate}
  \item if  $G,G'$ are distinguished by $\wl
    kt$ then they are distinguished by $\owl{k+1}t$;
  \item  if  $G,G'$ are distinguished by  $\owl{k+1}t$ then they are distinguished by $\wl{k}{t+1}$.
  \end{enumerate}
\end{theorem}

This theorem follows from the logical characterisation of $k$-WL
\cite{immlan90,caifurimm92}, which we will discuss below, but I am not
aware of an explicit reference or direct proof of the
result. Therefore, I find it worthwhile to include a proof as an appendix.

\begin{corollary}\label{cor:owl}
  Let $k\ge 1$. Then graphs $G,G'$ are distinguished by
  $\wl k{\infty}$ if and only if they are distinguished by
  $\owl{k+1}{\infty}$.
\end{corollary}

Let us now turn to the logical characterisation of the
WL-algorithm.

\begin{theorem}[\cite{caifurimm92}]\label{theo:wl-logic}
  Let $k\ge 2$ and $t\ge 0$. Then for all graphs $G,G'$ and tuples
  $\vec v\in V(G)^k,\vec v'\in V(G')^k$ the following are equivalent:
  \begin{eroman}
  \item
    $\owl kt(G,\vec v)=\owl kt(G',\vec v')$;
  \item for all formulas $\phi(\vec x)\in\LCkq kt$,
    \[
      G\models\phi(\vec v)\iff G'\models\phi(\vec v').
    \]
  \end{eroman}
\end{theorem}

We omit the proof, of this result. Intuitively,
a formula $\phi(x_1,\ldots,x_k)\in\LCkq k{t+1}$ is a Boolean
combination of atomic formulas and formulas of the form $\exists^{\ge
  p}x_i\psi(x_1,\ldots,x_k)$, where $\psi(x_1,\ldots,x_k)\in\LCkq kt$. The $i$th multiset in the definition of
$\owl k{t+1}$ takes care of these formulas.

From the Theorem~\ref{theo:wl-logic} and
Corollary~\ref{cor:owl}, we obtain the following characterisation of
the the indistinguishability of graphs. 

\begin{corollary}\label{cor:wl-logic}
  Let $k\ge 1$. The for all graphs $G,G'$ the following are
  equivalent:
  \begin{eroman}
  \item
    $\wl k{\infty}$ does not distinguish $G$ and $G'$;
  \item
    $\owl{k+1}{\infty}$ does not distinguish $G$ and $G'$;
  \item
    $G$ and $G'$ satisfy the same $\LCk{k+1}$-sentences.
  \end{eroman}
\end{corollary}

Combined with Proposition~\ref{prop:cr1wl},
Corollary~\ref{cor:wl-logic} also yields a characterisation of
distinguishability of graphs by colour refinement in terms of the
logic $\LCk2$, but only on the
levels of graphs. On the node level, $\LCk2$ is not the correct logic
to characterise colour refinement. Instead, it is the guarded fragment
$\LGCk2$. Intuitively, this is clear, because just like the guarded
logic, colour refinement only has access to the neighbours of a node.

\begin{theorem}\label{theo:cr-logic}
  Let $t\ge 0$. Then for all graphs $G,G'$ and vertices
  $v\in V(G),v'\in V(G')$ the following are equivalent:
  \begin{eroman}
  \item
    $\colref t(G,v)=\colref t(G',v')$;
  \item for all formulas $\phi(x)\in\LGCkq2t$,
    \[
      G\models\phi(v)\iff G'\models\phi(v').
    \]
  \end{eroman}
\end{theorem}

This theorem is obvious to researchers in the field, and it is
implicit in \cite{DBLP:conf/iclr/BarceloKM0RS20}, but for the reader's
convenience I have included a proof in the appendix.

\begin{remark}\label{rem:bin-wl}
  The definitions of $\wl kt$ and $\owl kt$ remain valid without any changes for
  arbitrary relational structures if we adapt the definition of atomic
  types as described in Remark~\ref{rem:bin-atp}.
  Theorems~\ref{theo:owl} and \ref{theo:wl-logic} as well as their
  corollaries also hold in this more general setting.

  There is one interesting observation that will be useful later: we
  can interpret $\owl kt$ over a graph (or binary structure) $G$ as
  $\wl1t$ over a structure $A_G$ with
  \begin{itemize}
  \item vertex set $V(A_G)\coloneqq
    V(G)^k$,
  \item binary relations $E_1(A_G),\ldots,E_k(A_G)$, where $E_i(A_G)$
    contains all pairs $(\vec v,\vec v')$ of tuples that differ
    exactly in the $i$-the component (that is $v_i\neq v_i'$ and
    $v_j=v_j'$ for $j\neq i$),
  \item unary relations $P_\theta(G)$ for all
    $\theta\in\{0,1\}^{2\binom{k}{2}+k\ell}$, where $\ell$ is the
    number of labels of $G$, containing all $\vec v$ with $\atp k(G,\vec
    v)=\theta$. 
  \end{itemize}
  Then it is easy to prove that for all graphs $G,G'$, all tuples
  $\vec v\in V(G)^k,\vec v'\in V(G')^k$, and all
  $t\in\Nat\cup\{\infty\}$ it holds that
 $
    \owl
    kt(G,\vec v)=\owl kt(G',\vec v')$ if and only if $\wl1t(A_G,\vec v)=\wl1t(A_{G'},\vec
    v')$.
  This observation goes back to \cite{ott97}.
  \uend
\end{remark}

\section{The Expressiveness of Feed-Forward Neural Networks}
\label{sec:dnn}

Before we discuss graph neural networks, we need to look into standard
fully-connected feed-forward neural networks (FNNs).
An \emph{FNN layer} of input dimension $m$ and output dimension $n$
computes a function from $\Real^m$ to $\Real^n$ of the form
\[
  \sigma\Big(A\vec x+\vec b\Big)
\]
for a \emph{weight matrix} $A\in \Real^{n\times m}$, a
\emph{bias vector} $\vec b\in\Real^{n}$, and an \emph{activation
  function} $\sigma:\Real\to\Real$ that is applied pointwise to
the vector $A\vec x+\vec b$.

An \emph{FNN} is a tuple $\big(L^{(1)},\ldots,L^{(d)})\big)$ of FNN layers, where the output dimension $n^{(i)}$ of $L^{(i)}$
matches the input dimension $m^{(i+1)}$ of $L^{(i+1)}$. Then the FNN computes a function from
$\Real^{m^{(1)}}$ to $\Real^{n^{(d)}}$, the composition
    $f^{(d)}\circ\cdots\circ f^{(1)}$ of the functions $f^{(i)}$ computed by
    the $d$ layers. 

    Typical activation functions used in practice are the \emph{logistic
  function} $\sig(x)\coloneqq\frac{1}{1+e^{-x}}$ (a.k.a. \emph{sigmoid
  function}), the \emph{hyperbolic tangent}
$\tanh(x)\coloneqq\frac{e^x-e^{-x}}{e^x+e^{-x}}$, and the
\emph{rectified linear unit} $\relu(x)\coloneqq\max\{x,0\}$. For
theoretical results, it is sometimes convenient to work with the
\emph{linearised sigmoid} $\lsig(x)\coloneqq\min\{\max\{x,0\},1\}$.
It is not important for us which activation function we use. We
only assume the activation function to be continuous. This implies
that functions computed by FNNs are continuous. Moreover, we want at
least some layers of our FNNs to have activation functions that are
not polynomial (in order for Theorem~\ref{theo:fnn-app} to apply).
Sometimes, we impose further restrictions.

In a machine learning setting, we only fix the shape or architecture
of the neural network and learn the weights. Formally, we say that an
\emph{FNN architecture} is a tuple
$(n^{(0)},\ldots,n^{(d)},\sigma^{(1)},\ldots,\sigma^{(d)})$ for
positive integers $d$ and $n^{(i)}$ and functions
$\sigma^{(i)}:\Real\to\Real$. We can view an FNN architecture as a
parametric model; the parameters are the entries of the weight
matrices $A^{(i)}\in \Real^{n^{(i)}\times n^{(i-1)}}$ and bias vectors
$\vec b^{(i)}\in\Real^{n^{(i)}}$ that turn the architecture into an FNN
defining a function $\Real^{n^{(0)}}\to\Real^{n^{(d)}}$. The
learning task is to set the parameters in a way that the resulting FNN
approximates an unknown target function
$f: \Real^{n^{(0)}}\to\Real^{n^{(d)}}$. Access to this unknown
function is provided through a set of \emph{labelled examples}
$(\vec x,f(\vec x))$. We cast the learning problem as a minimisation
problem: we minimise a \emph{loss function} that essentially measures the
difference between the function computed by the FNN and the target
function on the examples we are given. This is a difficult optimisation
problem, non-linear and not even convex, but it can be solved
remarkably well in practice using gradient-descent based algorithms.

So formally we distinguish between FNN architectures $\CN$, FNNs $N$
obtained from an FNN architecture by instantiating all the parameters
(the weight matrices and bias vectors), and the functions $f_N$
computed by FNNs $N$. For an \emph{FNN architecture} $\CN$, we let
$\CF_\CN$ be the class of all functions $f_N$ computed by the FNNs $N$
that we
obtain by instantiating the parameters of $\CN$.  For an FNN or FNN
architecture, we call $d$ (the number of layers) the \emph{depth}
and $\sum_{i=1}^d n^{(i)}$ the \emph{size}.  We call $n^{(0)} $ the
\emph{input dimension} and $n^{(d)}$ the \emph{output dimension},
and for $i\in[d]$ we call $n^{(i)}$ the dimension of layer
$i$. The \emph{width} of the network or network architecture is
$\max_{i\in[d]}n^{(i)}$.  

Our focus in this paper is not on learning but on
the \emph{expressiveness} of FNNs and other neural network
architectures, that is, on the question which functions can be
computed by such neural networks in principle. A fundamental
expressiveness result for FNNs is the following \emph{universal approximation
  theorem}.

\begin{theorem}[\cite{cyb89,hor91,leslinpinschoc93}]\label{theo:fnn-app}
  Let $\sigma:\Real\to\Real$ be continuous and not polynomial. Then
  for every continuous function $f:K\to\Real^n$, where
  $K\subseteq\Real^m$ is a compact set, and every $\epsilon>0$ there is a depth-$2$
  FNN $N$ with activation function $\sigma^{(1)}=\sigma$ on layer $1$ and
  no activation function on layer 2 (that is, $\sigma^{(2)}$ is the
  identity function) computing a function $f_{N}$ such that
  \[
    \sup_{\vec x\in K}\|f(\vec x)-f_{N}(\vec x)\|<\epsilon.
  \]
\end{theorem}

This seems to make neural networks extremely expressive. However, the
approximation of continuous functions on a compact domain implicitly
comes with a finiteness condition. Also, note that we put no
restriction on the dimension $n^{(1)}$ of the first layer; $n^{(1)}$ may
depend on the function $f$, and it may do so in a way that is
exponential in a description of $f$.

One way of formally capturing the limitations of FNNs is to look at
the VC dimension of families of sets decided
by such networks. Let $\CC\subseteq 2^{\mathbb X}$ be a set of subsets
of some set
$\mathbb X$ (in the case of FNNs, $\mathbb X=\Real^n$). We say that $X\subseteq \mathbb X$ is \emph{shattered} by
$\CC$ if for every subset $Y\subseteq X$ there is a $C\in\CC$ such
that $C\cap X=Y$. The \emph{Vapnik Chervonenkis
  (VC) dimension} \cite{vapche71} of $\CC$ is the maximum size of a finite set
shattered by $\CC$, or $\infty$ if this maximum does not exist. For an
FNN architecture $\CN$ of input dimension $n$ and output dimension $1$, we let
$\CC_\CN$ be the set of all sets $C_f\coloneqq \{\vec x\in\Real^n\mid f(\vec x)>0\}$
for
$f\in\CF_{\CN}$.

For a wide class of \emph{Pfaffian functions} that includes all the
activation functions discussed above, we have the following bound on
the VC dimension of FNN architectures.

\begin{theorem}[\cite{karmac97}]
  Let $\CN$ be an FNN architecture (of output dimension $1$) whose activation functions are
  Pfaffian functions. Then the VC dimension of $\CC_{\CN}$ is
  polynomial in the size of $\CN$ (where the polynomial depends on the
  specific activation functions used).
\end{theorem}

While it limits the expressiveness of FNNs, from the perspective of
machine learning, this theorem should be viewed as a positive result. The
reason is that
the number of labelled examples required to give statistical
approximation guarantees in the PAC model \cite{val84} for the FNN
learned from the examples can be bounded linearly in terms of the VC
dimension \cite{bluehrhauwar89}.

As a computation model, FNNs are similar to Boolean circuits, except
that they do ``analog'' computations using arbitrary real numbers as
weights. The question arises if such analog circuits are more
expressive than Boolean circuits. In general, they are, but
surprisingly they are not if we restrict the activation functions that
can be used to be piecewise polynomial. It
can be shown that properties decided by families of FNNs of
bounded depth and polynomial size with piecewise polynomial activation
functions belong to
$\TC^0$ \cite{maa97}. Recall that
$\TC^0$ is the class of all properties decidable
by a family of Boolean threshold circuits of bounded depths and
polynomial size. In threshold circuits, we have threshold gates that
output $1$ if at least $t$ of their inputs are $1$, for some $t$.

\section{Graph Neural Networks}
\label{sec:gnn}
Let us consider a machine learning scenario on graphs where we want to
learn a function defined on graphs, vertices, or tuples of
vertices. We usually think of a supervised learning
scenario, where we want to learn an unknown function from labelled
examples, but an equally important scenario is that of learning a
representation of graphs or their vertices. \emph{Crucially, we want these
  functions to be invariants.}\footnote{To avoid confusion here with
  respect to the usage of invariant and equivariant neural networks in the machine learning literature, recall that we
  defined a \emph{vertex invariant} to be an \emph{equivariant}
  function that associates a function defined on $V(G)$ with each
  graph $G$, and more generally a \emph{$k$-ary invariant} to be an equivariant
  function that associates a function defined on $V(G)^k$ with each
  graph $G$.}

In both the supervised and representation learning scenarios we
``learn'' models representing the functions, and we want to guarantee
that the functions represented by our models are invariants. We could
train FNNs taking as inputs the adjacency matrices of graphs, but 
it would be difficult to ensure invariance. Graph neural networks
(GNNs) are neural network architectures that guarantee invariance by
their design.

Many different versions of the GNNs have been considered in the
literature (see, for example, \cite{galmic10,DBLP:conf/icml/GilmerSRVD17,hamyinles17,kipwel17,scagortso09}). Here, we focus on what seems
to be a standard, core model. It sometimes goes under the name
``message passing neural network (MPNN)'' \cite{DBLP:conf/icml/GilmerSRVD17} or ``aggregate-combine graph neural network (AC-GNN)'' \cite{DBLP:conf/iclr/BarceloKM0RS20}.

\subsection{GNNs as a Computation Model}
A \emph{GNN layer} of input dimension $p$ and output dimension $q$  is
defined by specifying two functions: an \emph{aggregation function}
$\agg$ mapping finite multisets of vectors in $\Real^{p}$ to vectors
in $\Real^{p'}$, and a \emph{combination function}
$\comb:\Real^{p+p'}\to\Real^q$.
A GNN is a tuple $(L^{(1)},\ldots,L^{(d)})$ of GNN layers, where the output dimension $q^{(i)}$ of $L^{(i)}$
matches the input dimension $p^{(i+1)}$ of $L^{(i+1)}$. We call
$q^{(0)}\coloneqq p^{(1)}$ the input dimension of the GNN and
$q^{(d)}$ the output dimension. 

To define the semantics, recall that we call mappings $\zeta:V(G)\to\Real^p$
$p$-dimensional \emph{feature
  maps}.
Let $L$ be a GNN layer of input dimension $p$ and output dimension
$q$ with aggregation function $\agg$ and combination function
$\comb$. Applied to a graph $G$, it transforms a $p$-dimensional feature map $\zeta$ to
a $q$-dimensional feature map $\eta$ defined by
\begin{equation}
  \label{eq:2}
    \eta(v)\coloneqq \comb\Big(\zeta(v),\agg\big(\biglmulti \zeta(w)\bigmid w\in N^G(v)\bigrmulti\big)\Big).
\end{equation}
Computationally, it may be useful to think of the features $\zeta(v)$
of the vertices as states. Then the feature transformation computed by the GNN layer is based
on a simple distributed message passing protocol whose computation
nodes are the vertices of our graph: each node $v$ sends its state $\zeta(v)$ to all its
neighbours. Upon receiving the state $\zeta(w)$ of its neighbours,
node $v$ applies the aggregation function $\agg$ to the multiset of
these states, yielding a vector $\alpha(v)$ and then uses the combination function $\comb$
to compute the new state $\eta(v)$ from $\zeta(v)$ and $\alpha(v)$.

Observe that this transformation is isomorphism invariant: for graphs
$G,G'$ and feature maps $\zeta\in\CF_p(G),\zeta'\in\CF_p(G')$ transformed by $L$
to feature maps $\eta,\eta'$, if $f$ is
an isomorphism from $G$ to $G'$ satisfying $\zeta(v)=\zeta'(f(v))$ for all
$v\in V(G)$, then  $\eta(v)=\eta'(f(v))$ for all
$v\in V(G)$.

A GNN consisting of layers $L^{(1)},\ldots,L^{(d)}$ composes the feature transformations computed by all its
layers. Thus if the input dimension of $L^{(0)}$ is $q^{(0)}$ and the
output dimension of $L^{(d)}$ is $q^{(d)}$, on a graph $G$ it
transforms $q^{(0)}$-dimensional feature maps $\zeta^{(0)}$ to
$q^{(d)}$-dimensional feature maps
$\zeta^{(d)}$. We usually denote the
intermediates feature maps obtained by applying the layers
$L^{(0)},\ldots,L^{(t)}$ to the initial feature map $\zeta^{(0)}$ by
$\zeta^{(t)}$. Then clearly, the transformation
$\zeta^{(0)}\mapsto\zeta^{(t)}$ is isomorphism invariant for all $t$.

There is an alternative mode of operation of GNNs where we do not have
a fixed number of layers but repeatedly apply the same layer. A
\emph{recurrent GNN} consists of a single GNN layer with the
same input and output dimension $q$. Applied to an initial
$q$-dimensional feature map $\zeta^{(0)}$, it yields an infinite sequence
$\zeta^{(1)},\zeta^{(2)},\ldots$ of $q$-dimensional feature maps,
where $\zeta^{(t+1)}$ is obtained by applying the GNN layer to
$(G,\zeta^{(t)})$. Again, the transformation
$\zeta^{(0)}\mapsto\zeta^{(t)}$ is isomorphism invariant for all
$t$. Note that we do not impose any convergence requirements on the
sequence $\zeta^{(1)},\zeta^{(2)},\ldots$; typically, there will be no convergence.

In the end, we usually do not want to compute feature transformations
but functions defined on graphs (graph invariants) or on the vertices
of graphs (vertex invariants). The features are only supposed to be
internal representations. Towards this end, we choose the initial
feature map $\zeta^{(0)}$ to represent the labelling of the input
graph $G$. That is, we let $\zeta^{(0)}\coloneqq\col(G)$ possibly
padded by $0$s to reach the input dimension of our GNN. We always
assume that the input dimension $q^{(0)}$ of the GNN is at least the
label number $\ell$ of the input graphs. In Section~\ref{sec:ri} we
shall discuss GNNs with random initialisation, where we pad the
initial feature with random numbers.

If we want to define a vertex invariant with
range $\mathbb Y$, we define a readout function
$\ro:\Real^{q^{(d)}}\to\mathbb Y$ that we apply to the final feature
  $\zeta^{(d)}(v)$ of every node. 
For recurrent GNNs, we need to fix the number $d=d(G)$ of iterations
we take; it may depend on the input graph. Then the (recurrent) GNN
maps a graph $G$ to $\ro(\zeta^{(d)})$. The readout function
$\ro$ is part of the specification of a GNN (or GNN architecture), and
for a recurrent GNN the function $d$ also needs to be specified.

If we want to use our GNN to compute a graph invariant, say, also with
range $\mathbb Y$, we need to
specify an \emph{aggregate readout function} $\aggro$ mapping finite
multisets of vectors in $\Real^{q^{(d)}}$ to $\mathbb Y$, and we
apply $\aggro$ to the multiset of all nodes' output features. So
the GNN maps $G$ to
\[
  \aggro\Big(\biglmulti \zeta^{(d)}(v)\bigmid v\in
  V(G)\bigrmulti\Big).
\]
Again, for recurrent GNNs we need to specify $d=d(G)$.

\subsection{Learning GNNs}

As defined so far, GNNs are just distributed
algorithms. What turns them into neural network models is that the
aggregation, combination, and readout functions are learned functions
represented by neural networks.

It is clear how to do this for the combination and readout
functions. For the theoretical analysis carried out in this paper it
suffices to assume that $\comb$ and $\ro$ are computed by FNNs of
depth 2 with a non-polynomial activation function on the interior
layer. Then by Theorem~\ref{theo:fnn-app}, we can approximate
arbitrary continuous functions by $\comb$ and $\ro$. In the
literature, it is sometimes assumed that $\comb$ is computed by a single FNN layer, for
example, in the ``simple GNNs'' of
\cite{DBLP:conf/iclr/BarceloKM0RS20}. We do not make this
assumption here. Several FNN layers in $\comb$ can always be simulated
by several ``simple'' GNN layers, so this does not change the
expressiveness of the model. However, allowing several layers and
thereby a direct application of the universal approximation property
of FNNs makes our analysis simpler and cleaner.

Next, let us look at the aggregation function $\agg$. Here we have the
difficulty that it is not defined on vectors, but on multisets of
vectors. In general, we can think of the aggregation function as being
composed of three functions:
\begin{itemize}
\item
  a ``message function''
$\logic{msg}:\Real^p\to\Real^{p''}$ that takes the states of the neighbours and
distills from them the messages send,
\item
an actual aggregation
function $\logic{aa}$ from finite multisets over $\Real^{p''}$ to
$\Real^{p''}$ that combines the multiset of messages received into a
single vector,
\item a post-processing function $\logic{pp}:\Real^{p''}\to\Real^{p'}$.
  which is passed on to the combination function.
\end{itemize}
  The aggregation $\logic{aa}$ is
usually a fixed function, typically the sum, maximum, or
arithmetic mean of the elements of the multiset. We always use
summation as aggregation. It has been proved in \cite{xuhulesjeg19} that
summation is more expressive than max or mean and sufficient
to realise the full expressiveness of GNNs (at least from the
perspective from which we analyse GNNs here). The message function
$\logic{msg}$ and
post-processing function $\logic{pp}$ can be learned functions; we could
assume them to be represented by a FNNs. However, at least in theory
we do not need the functions $\logic{msg}$ and $\logic{pp}$ at all, because we can merge $\logic{msg}$ into the combination
function of the previous GNN-layer (formally, by writing the message
to designated coordinates of the feature vector), and we can merge
$\logic{pp}$ into the combination function of the current layer.
Thus, in the following, we assume that $\agg$ is simply the summation
function. Then a single GNN layer (cf.~\eqref{eq:2}) maps a feature map $\zeta$ to the
feature map $\eta$ defined by
\begin{equation}
  \label{eq:3}
   \eta(v)\coloneqq \comb\Big(\zeta(v),\sum_{w\in N^G(v)}\zeta(w)\Big),
\end{equation}
where $\comb$ is a function represented by an FNN.

Let me remark that in practice, one
typically uses a learned linear messaging function $\logic{msg}$, so the summation
in \eqref{eq:3} becomes $\sum_{w\in N^G(v)}A\zeta(w)$ for
a weight matrix $A$.

Finally, if we need an aggregate readout function $\aggro$, we take it
to be the summation of the feature vectors in the multiset it receives
followed by some post-processing function represented by an FNN.

Having specified the functions $\comb,\agg,\ro,\aggro$ this way, we can
now describe their FNN architectures, and we can combine these into a
\emph{GNN architecture}. Then we can learn the parameters of such a GNN
architecture from labeled examples using the same techniques as for
FNNs. However, as for FNNs, we are mainly concerned with
expressiveness questions here.

\subsection{Variants}\label{sec:gnn-variants}
Numerous variants of our basic message passing GNN model have been
suggested in the literature; I refer the reader to the recent survey
\cite{DBLP:journals/tnn/WuPCLZY21}. Let me briefly discuss a few
variants that will be useful later in this paper.

We can define GNNs not only for undirected vertex-labelled graphs, as
we do here, but also for directed graphs with possibly labelled
edges, that is, for arbitrary binary relational structures. The
easiest way to adapt GNNs to this setting is by introducing an explicit
(learned) message function $\logic{msg}$ that not only depends on the current state
of the source node of the message, but also on the type of edges and the direction they are traversed (see, for example, \cite{schlikipblo+18,tonritwolgro20+}). 

A second addition to the basic model that will be relevant for us is
that of a \emph{global readout}. As defined, the information a node
receives during a GNN computation is inherently local. For example, a node can
never receive any information about a node in a different connected
component of the input graph. What we can do to mitigate this weakness
is to provide aggregate global information to each node. The easiest
way to do 
this is by modifying the update function \eqref{eq:3} to 
\begin{equation}
  \label{eq:4}
   \eta(v)\coloneqq \comb\Big(\zeta(v),\sum_{w\in N^G(v)}\zeta(w), \sum_{w\in V(G)}\zeta(w)\Big).
\end{equation}
Following \cite{DBLP:conf/iclr/BarceloKM0RS20}, we call GNNs with this
form of updates \emph{GNNs with global readout}. An essentially
equivalent mechanism to distribute global information is adding a \emph{virtual node} connected to all other nodes of
the graph \cite{DBLP:conf/icml/GilmerSRVD17}.

A third variant of basic GNNs, \emph{higher order GNNs}, will be
discussed in Section~\ref{sec:higher-order}.

\section{Weisfeiler and Leman Go Neural}
\label{sec:wl2gnn}

In this section, we shall prove that GNNs have exactly the same
expressiveness as colour refinement when it comes to distinguishing
graphs or their vertices. 

\begin{theorem}[\cite{morritfey+19},\cite{xuhulesjeg19}]\label{theo:gnn-upper}
  Let $d\ge 1$, and let $\xi$ be a vertex invariant computed by an
  $d$-layer GNN. Then $\colref d$ refines $\xi$, that is, for all graphs $G,G'$ and vertices $v\in
  V(G),v'\in V(G')$,
  \[
    \colref{d}(G,v)=\colref{d}(G',v')\implies\xi(G,v)=\xi(G',v').
  \]
\end{theorem}

This theorem holds regardless of the choice of the aggregation
function $\agg$ and the combination function $\comb$.

\begin{proof}
  For $0\le t\le d$, let $\zeta^{(t)}$ be the feature map computed
  by the first $t$ layers of the GNN on $G$ and $G'$. We assume that
  $V(G)$ and $V(G')$ are disjoint and use $\zeta^{(t)}$ to denote the
  union of the feature maps on the two graphs. To further simplify the notation, for $v\in V(G)\cup V(G')$ we write $\colref{t}(v)$ instead of
  $\colref{t}(G,v)$ or $\colref{t}(G',v)$.

  We shall prove by induction on $t$
  that $\colref{t}$ refines $\zeta^{(t)}$. As
  $\xi(G,v)=\ro(\zeta^{(d)}(v))$, this will imply the theorem.

  For the base
  step, recall that the initial feature map $\zeta^{(0)}$ as well as
  $\colref0=\col$ just encode
  the label information of the graph. So let us consider the
  inductive step $t\to t+1$. By the
  induction hypothesis, $\colref{t}$
  refines $\zeta^{(t)}$. 
  Let $v\in V(G),v'\in V(G')$ such that
  $\colref{t+1}(v)=\colref{t+1}(v')$. Then
  $\colref{t}(v)=\colref{t}(v')$ and
  \[
    \biglmulti \colref{t}(w)\bigmid w\in N^{G}(v)\bigrmulti
    =\biglmulti \colref{t}(w')\bigmid w\in N^{G'}(v')\bigrmulti.
  \]
  Thus, by the induction hypothesis,
  $\zeta^{(t)}(v)=\zeta^{(t)}(v')$ and 
  \[
    \underbrace{\biglmulti \zeta^{(t)}(w)\bigmid w\in N^{G}(v)\bigrmulti}_{=:M}
    =\underbrace{\biglmulti \zeta^{(t)}(w')\bigmid w\in N^{G'}(v')\bigrmulti}_{=:M'}.
  \]
  It follows that
  \begin{align*}
    \zeta^{(t+1)}(v)&=\comb\big(\zeta^{(t)}(v),\agg(M)\big)\\
                       &=\comb\big(\zeta^{(t)}(v'),\agg(M')\big)=\zeta^{(t+1)}(v').\qedhere
  \end{align*}
\end{proof}

\begin{corollary}\label{cor:upper1}
  Let $\xi$ be a vertex invariant computed by a recurrent
  GNN. Then $\colref{\infty}$ refines $\xi$.
\end{corollary}

\begin{corollary}\label{cor:upper2}
  Let $\xi$ be a graph invariant computed by a GNN (recurrent or
  not). Then for all graphs $G,G'$, if $\xi(G)\neq\xi(G')$ then
  $\colref{\infty}$ distinguishes $G$ and $G'$.
\end{corollary}

Let us now turn to the converse. The following theorem is a slight
strengthening of a theorem due to \cite{morritfey+19,xuhulesjeg19} in
that we find a single recurrent GNN that captures all colour
refinement iterations. The result is still not uniform because we
need GNNs depending on the size of the input graphs.

\begin{theorem}[\cite{morritfey+19,xuhulesjeg19}]\label{theo:gnn-lower}
  Let $n\in\Nat$. Then there is a recurrent GNN such that for all
  $t\le n$, the vertex invariant $\xi^{(t)}$ computed in the $t$-th
  iteration of the GNN refines $\colref t$ on all graphs of order at
  most $n$.

  That is, for all graph $G,G'$ of order at most $n$ and all vertices
  $v\in V(G),v'\in V(G')$:
  \[
    \xi^{(t)}(G,v)=\xi^{(t)}(G',v')\implies \colref{t}(G,v)=\colref{t}(G',v').
  \]
 \end{theorem}

The theorem even holds for the restricted GNN model with sum-aggregation,
that is, with \eqref{eq:3} as update operation.

In the proof of the theorem, we follow \cite{xuhulesjeg19}. We need
the following lemma.

\begin{lemma}[\cite{xuhulesjeg19}]\label{lem:sum}
  Let $m\in\Nat$, and let $X\subseteq(0,1)$ be a nonempty finite set. Then there
  is a function $f:X\to(0,1)$ with the following two properties:
  \begin{enumerate}
  \item for all multisets
    $M\subseteq X$ of order at most $m-1$ we have
    $\sum_{x\in M}f(x)<\min X$;
  \item for all multisets
    $M,M'\subseteq X$ of order at most $m-1$, if $M\neq M'$ then
    $\sum_{x\in M}f(x)\neq\sum_{x\in M'}f(x)$.
  \end{enumerate}
\end{lemma}

\begin{proof}
  Let $x_1,x_2,\ldots,x_n$ be an enumeration of $X$, and let
  $\epsilon\coloneqq \min X$. Choose
  $k\in\Nat$ such that $m^{-k}\le\epsilon$ and define $f$ by
  $f(x_i)\coloneqq m^{-k-i}$. 

  Let $M\subseteq X$ be a multiset of order at most
  $m-1$, and let $a_i$ be the multiplicity of $x_i$ in $M$. Then
  $a_i<m$. We have $\sum_{x\in M}f(x)=\sum_{i=1}^n a_i
  m^{-k-i}$, or
  \[
    0.\underbrace{0\ldots 0}_{k\text{ times}}a_1a_2\ldots a_n.
  \]
  in $m$-ary representation. It is clear that these numbers are
  strictly between $0$ and $\epsilon$, which implies (1), and
  that they are
  distinct for distinct multisets $M,M'$, which implies (2).
\end{proof}

\begin{proof}[Proof of Theorem~\ref{theo:gnn-lower}]
  Without loss of generality, we assume that $n\ge 1$.
  In this proof, we assume all graphs to be of order at most
  $n$. Recall that the range of $\col(G)$ is
  $\{0,1\}^\ell$. 

  Let $g_0$ be an arbitrary injective mapping from $\{0,1\}^\ell$ to
  $(0,1)$, and let $X_0\coloneqq g_0(\{0,1\}^\ell)$.

  Now suppose that for some $t\ge0$ we have defined a finite set
  $X_{t}\subseteq(0,1)$. We choose a mapping $f_t:X_t\to (0,1)$
  according to Lemma~\ref{lem:sum} with $m\coloneqq 2n$ and
  $X\coloneqq X_t$, and we let
  $X_{t+1}$ be the set of all numbers $\sum_{x\in M} f_t(x)$, where
  $M\subseteq X_t$ is a multiset of order at most
  $2n-1$. Note that $X_{t+1}\cap X_t=\emptyset$ by Lemma~\ref{lem:sum}(1).

  Let $Z_0\coloneqq\{0,1\}^\ell$ and $Z_t\coloneqq
  X_t\times\{0\}^{\ell-1}$ for $t\in[n]$. Then the $Z_t$ are mutually disjoint
  finite subsets of $\Real^\ell$. Let $g:\Real^\ell\to\Real^\ell$ be
  a continuous function such that
  $g(\vec z) = (f_0(g_0(\vec z)),0,\ldots,0)$ for $\vec z\in Z_0$ and $g(\vec
  z)=(f_t(z_1),0,\ldots,0)$ for $\vec z=(z_1,\ldots,z_\ell)\in Z_t$
  and $t\in[n]$. (We can always find a continuous function with given
  values on a finite set of arguments.)

  Suppose first we have a recurrent GNN of dimension $\ell$ with
  sum-aggregation and a combination function
  \[
    \comb(\vec x,\vec y)
    \coloneqq g\left(\vec x+2\vec y\right).
  \]
  (this combination function cannot necessarily be represented by an FNN, but
  we will fix that later). 
   Let $G$ be a
  graph, and let $\zeta^{(0)}=\col(G)$, and let $\zeta^{(1)},\zeta^{(2)},\ldots$ be
  the the sequence of feature maps computed by this GNN. Then for all
  $t\in[n]$ and $v\in V(G)$ we have
  \[
    \zeta^{(t)}(v)=g\left(\zeta^{(t-1)}(v)+2\sum_{w\in
        N^G(v)}\zeta^{(t-1)}(w)\right).
  \]
  This can be written as $g\left(\sum_{\vec z\in M}\vec z\right)$ for the
    multiset
    \[
      M\coloneqq\biglmulti\zeta^{(t-1)}(v)\bigrmulti\cup\biglmulti
      \zeta^{(t-1)}(w), \zeta^{(t-1)}(w)\bigmid w\in N^G(v)\bigrmulti,
    \]
    where the union of multisets adds the multiplicities of the
    elements in the two sets. Note that  $\zeta^{(t-1)}(v)$ can be
    retrieved from $M$ as the only elements of odd multiplicity.

    Now a
    straightforward induction shows that for all $t\in[n]$ and $v\in V(G)$
    we have $\zeta^{(t)}(v)=(f_t(x),0,\ldots,0)$ for some $x\in X_t$. This
    implies that $\zeta^{(t)}(v)+2\sum_{w\in
        N^G(v)}\zeta^{(t)}(w)=\zeta^{(t)}(v')+2\sum_{w'\in
        N^G(v')}\zeta^{(t)}(w')$ if and only if
      $\zeta^{(t)}(v)=\zeta^{(t)}(v')$ and $\biglmulti
      \zeta^{(t)}(w)\bigmid w\in N^G(v)\bigrmulti=\biglmulti
      \zeta^{(t)}(w')\bigmid w'\in N^G(v')\bigrmulti$. In fact, this
      equality holds across graphs, that is, even if $v'$ is from a
      different graph $G'$. From this, we can derive
      \[
        \zeta^{(t)}(v)=\zeta^{(t)}(v')\implies \colref{t}(v)=\colref{t}(v')
      \]
      by another simple induction. (To simplify the
      notation, we omit an explicit reference to the graphs here.)

      This does not yet prove the theorem, because in our GNN we used
      the function $g$ in the definition of the the combination
      function $\comb$, and in general $g$ cannot be represented by an
      FNN. However, using Theorem~\ref{theo:fnn-app}, we can
      approximate $g$ by a 2-layer FNN on the compact set
      $[0,1]^\ell$, which contains all the $Z_t$ that are relevant for
      us. With a sufficiently close approximation and some
      $\epsilon$-$\delta$ magic, the argument still goes
      through. \end{proof}

\begin{remark}
  Theorem~\ref{theo:gnn-lower} (in slightly different versions) was proved independently in
  \cite{morritfey+19} and \cite{xuhulesjeg19}.
  It is interesting to compare the two versions of the theorem and
  their proofs. In a nutshell, we observe a tradeoff between generality and
  efficiency.

  The proof 
  from \cite{xuhulesjeg19}, which we presented here, is simpler and applies to a larger
  class of activation functions. The proof in
  \cite{morritfey+19} constructs the FNNs involved in the combination
  function of the GNN explicitly and does not use the universal approximation
  theorem. For this reason, it allows for a better control of the
  complexity of the FNN: it only requires a single layer FNN that is
  guaranteed to be of polynomial size, whereas the FNN
  we get out of the universal approximation theorem needs two layers
  and may be exponentially large in $n$. This means that \cite{morritfey+19} constructs a GNN of polynomial size that computes the
  colour refinement partitions. (The proof of \cite{xuhulesjeg19} does
  not yield this result.)   As for the encoding of the 
  colour information: \cite{morritfey+19} uses integer vectors of 
  linear size with numbers of linear bit-length, the proof from 
  \cite{xuhulesjeg19} we present here uses rational numbers of 
  exponential bit length. 

  The size of the GNNs and related parameters like depth and width, which directly affect the complexity of
  inference and learning, 
  definitely require close attention; several of the open questions
  stated in Section~\ref{sec:conc} revolve around these complexity
  theoretic issues.
   \uend 
\end{remark}

\begin{corollary}\label{cor:lower1}
    Let $n\ge1$. Then there is a recurrent GNN such
  that for all graphs $G,G'$ of order at most $n$, if $\colref{\infty}$ distinguishes $G$ and $G'$ then
  $\xi(G)\neq\xi(G')$, where $\xi$ is the graph invariant computed by
  the GNN in $n$ iterations.
\end{corollary}

As $\wl1t$ refines $\colref t$, in Theorem~\ref{theo:gnn-upper} and its
corollaries, we can replace colour refinement with the 1-dimensional WL
algorithm. We can also do this in Corollary~\ref{cor:lower1}, because
by Proposition~\ref{prop:cr1wl} there is no difference in the power of $\wl1t$ and $\colref t$ on the
graph level. However, Theorem~\ref{theo:gnn-lower} needs to be
modified for $1$-WL.

\begin{theorem}\label{theo:gnn-wl-lower}
  Let $n\ge1$. Then there is a recurrent GNN with global
  readout such
  that for all $t\in[n]$, the vertex invariant $\xi^{(t)}$ computed in the $t$-th
  iteration of the GNN refines $\colref t$ on all graphs of order at
  most $n$.
\end{theorem}

The proof of this Theorem is an easy modification of the proof of
Theorem~\ref{theo:gnn-lower}.

There is also a corresponding version of Theorem~\ref{theo:gnn-upper}
with $1$-WL and GNNs with global readout. While these results are not
explicitly stated elsewhere, it is fair to say that the main insight
underlying them is from \cite{DBLP:conf/iclr/BarceloKM0RS20}.

\section{The Logical Expressiveness of GNNs}
\label{sec:gnn-logic}

The previous section's results fully characterise the
distinguishing power of GNNs, both on the graph level and on the
vertex level. However, the results are non-uniform because
the GNNs depend on the size of the input graphs. They do not tell us
which functions defined globally on all graphs are expressible by
GNNs. In this section, we take a first step towards characterising
these functions: we obtain a characterisation of all first-order queries expressible by GNNs.

A \emph{$k$-ary query} is a $k$-ary invariant with range
$\{0,1\}$. We can view $0$-ary queries, also called \emph{Boolean queries}, as isomorphism closed classes
of graphs and $k$-ary queries for $k\ge 1$ as equivariant mappings
$Q$ that map each graph $G$ to a set $Q(G)\subseteq V(G)^k$. 
A formula $\phi(\vec x)$ for some logic $\LL$ \emph{expresses} a
$k$-ary query
$Q$ if for all graphs $G$ and $\vec v\in V(G)^k$ we have $G\models
\phi(\vec v)\iff\vec v\in Q(G)$. We are mainly concerned with unary queries
expressible in first-order logic here.

We say that a GNN \emph{expresses} a unary query if it approximates
the corresponding vertex invariant. Let us make this precise as
follows.  Suppose that $\xi$ is the vertex invariant computed by the
GNN. Then we say that the GNN expresses $Q$ if there is an
$\epsilon<1/2$ such that for all graphs $G$ and vertices $v\in V(G)$,
\begin{equation}
  \label{eq:7}
  \begin{cases}
    \xi(G,v)\ge 1-\epsilon&\text{if }v\in Q(G),\\
    \xi(G,v)\le\epsilon&\text{if }v\not\in Q(G).
  \end{cases}
\end{equation}
Similarly, we can define a GNN computing a graph invariant to express
a Boolean query.

\begin{theorem}[\cite{DBLP:conf/iclr/BarceloKM0RS20}]\label{theo:gnn-logic}
  Let $\CQ$ be a unary query expressible in graded modal logic
  $\LGCk2$. Then there is a GNN that expresses $\CQ$.
\end{theorem}

In \cite{DBLP:conf/iclr/BarceloKM0RS20}, the theorem is only proved
for GNNs that use a linearised sigmoid $\lsig$ as activation
function. The proof can easily be adapted to the $\relu$-activation
function, but it is not obvious how to adapt it to other standard activation
functions such as $\sig,\tanh$. The reason is that due to the
non-uniformity, we do not have an obvious compact domain where we can
apply the universal approximation property of FNNs. We sketch a proof
of the theorem using $\lsig$.

\begin{proof}[Proof sketch]
  Let $\phi(x)$ be a $\LGCk2$-formula expressing $\CQ$. Let
  $\psi_1(x),\ldots,\psi_d(x)=\phi(x)$ be a list of all subformulas of
  $\phi(x)$, where we view guarded quantification as a single operation,
  that is, for $\psi(x)=\exists^{\ge i}y(E(x,y)\to\psi'(y))$ we only
  consider the subformula $\psi'(y)$. We assume that the $\psi_i(x)$
  are sorted in a way compatible with the subformula order, that is,
  if $\psi_i$ is a proper subformula of $\psi_j$ then
  $i<j$. Furthermore, we assume that the first $\ell$ formulas in the
  list are the label atoms $P_i(x)$.

  We design a GNN with $d-\ell$ layers, where the $t$-th layer has
  input dimension $q^{(t-1)}\coloneqq \ell+t-1$ and output dimension
  $\ell+t$. During the evaluation of the GNN on a graph $G$, the
  feature map $\zeta^{(t)}$ is supposed to map each vertex $v$ to a
  vector
  $\zeta^{(t)}(v)=(z_1,\ldots,z_{\ell+t})\in\{0,1\}^{\ell+t}$,
  where $z_i=1$ if and only if $G\models\psi_i(v)$. It is easy to
  define a combination function $\comb^{(t)}$ achieving this. For
  example, if
  $
    \psi_t= \exists^{\ge p}y(E(x,y)\to\psi_s(y))
  $
  for some
  $s<t$, we must define
  $\comb^{(t)}:\Real^{\ell+t-1}\times\Real^{\ell+t-1}\to\Real^{\ell+t}$
  in such a way that it maps $\Big(\zeta^{(t-1)}(v),\sum_{w\in
    N^G(v)}\zeta^{(t-1)}(w)\Big)$ to the vector
  $(z_1,\ldots,z_{\ell+1})$, where $z_i=(\zeta^{(t-1)}(v)\big)_i$ for
  $i\in[\ell+t-1]$ and
  \[
    z_{\ell+t}=\begin{cases}
      1&\text{if }\sum_{w\in
        N^G(v)}\big(\zeta^{(t-1)}(w)\big)_s\ge p,\\
      0&\text{otherwise}.
    \end{cases}
  \]
  To achieve this, we can let $\comb^{(t)}=\lsig(A\vec x+b)$, where
  \begin{itemize}
  \item $A$ is the $(\ell+t)\times(2(\ell+t-1))$-matrix with entries
  $A_{ii}=1$, $A_{ij}=0$ for
  $i\in[\ell+t-1],j\in[2(\ell+t-1)]\setminus\{i\}$ and
  $A_{(\ell+t)(\ell+t-1+s)}=1$, $A_{(\ell+t)j}=0$ for
  $j\in[2(\ell+t-1)]\setminus\{\ell+t-1+s\}$,
  \item $\vec b$ is the vector with entries $b_i=0$ for
    $i\in[\ell+t-1]$ and $b_{\ell+t}=-p+1$.
    \qedhere
  \end{itemize}
\end{proof}

Clearly, the converse of the theorem does not hold. For example, we
can use a GNN to express the unary query ``vertex $v$ has twice as
many neighbours with label $P_1$ as it has neighbours with label
$P_2$'', which is not expressible in graded modal logic. However, the
theorem has an interesting partial converse.

\begin{theorem}[\cite{DBLP:conf/iclr/BarceloKM0RS20}]\label{theo:logic-gnn}
  Let $\CQ$ be a unary query expressible by a GNN and also
  expressible in first-order logic. Then $\CQ$ is expressible in
  $\LGCk2$.
\end{theorem}

The proof of this result is based on a characterisation of
$\LGCk2$ as the fragment of first-order logic invariant under
counting bisimulation~\cite{DBLP:journals/corr/abs-1910-00039}.

Unsurprisingly, GNNs with global readout can express all
properties that are expressible in the logic $\LCk2$. This corresponds
precisely to the transition from colour refinement to $1$-WL enabled
by global readout in the previous section.

\begin{theorem}[\cite{DBLP:conf/iclr/BarceloKM0RS20}]
  Let $Q$ be a Boolean or unary query expressible in $\LCk2$. Then there is
  a GNN with global readout that expresses $Q$.
\end{theorem}

\section{Higher Order GNNs}
\label{sec:higher-order}

Inspired by the correspondence between 1-WL and GNNs, Morris et
al.~\cite{morritfey+19} proposed \emph{higher-order GNNs}, a deep
learning architecture with an expressiveness corresponding to
$k$-WL. The idea is to use the oblivious WL-version because oblivious
WL on a graph $G$ essentially operates on a binary structure $A_G$
with vertex set $V(G)^k$ (as we have seen in
Remark~\ref{rem:bin-wl}). We can define a \emph{$k$-GNN} operating on
a graph $G$ to be a GNN operating on $A_G$, using the extension
of GNNs to binary structures described in
Section~\ref{sec:gnn-variants}.
It is important to note that nodes of the message passing network
carrying out the $k$-GNN computation are $k$-tuples of vertices.
We obtain the following theorem
directly as a corollary to our earlier results.

\begin{theorem}[\cite{morritfey+19,DBLP:conf/nips/0001RM20}]\label{theo:kgnn}
  Let $k\ge 2$.
  \begin{enumerate}
  \item
    Let $d\ge 1$, and let $\xi$ be a $k$-ary invariant computed by an
    $d$-layer $k$-GNN. Then $\owl kd$ refines $\xi$.
\item 
    For all $n\ge 1$ there is a recurrent $k$-GNN such
    that for all $t\le n$ the vertex invariant $\xi^{(t)}$ computed by
    the $t$-th iteration of the GNN refines $\owl kt$.
\end{enumerate}
\end{theorem}

It is important to note that $k$-GNNs correspond to $k$-dimensional
oblivious WL and hence to $(k-1)$-dimensional WL; this can easily lead
to confusion.

The version of $k$-GNNs described in \cite{morritfey+19} is slightly
different. In particular, there is also a version that operates on
$k$-element sets rather than $k$-tuples (which saves some
memory). While there may be practical considerations leading to these
alternative approaches, I believe the theoretical essence of $k$-GNNs
is most transparent in the version we describe here.

\begin{remark}
  Since $n$-dimensional Weisfeiler-Leman characterises graphs of order
  $n$ up to isomorphism, we can use $n$-GNNs as a universal invariant
  neural network architecture that is able to approximate all invariant
  and equivariant functions defined on graphs of order at most $n$. In
  fact, the higher-order GNNs are closely related to the invariant and
  equivariant graph networks introduced in
  \cite{DBLP:conf/nips/MaronBSL19,DBLP:conf/iclr/MaronBSL19}.

  Interestingly, for many restricted graph classes, for example, all
  classes of graphs excluding a fixed graph as a minor~\cite{gro17}, a
  constant order is already sufficient for the universality. In
  particular, all planar graph invariants can be expressed by
  $4$-GNNs. This follows from the fact that $3$-WL characterises all
  planar graphs up to isomorphism \cite{kieponschwe17}.
  \uend
\end{remark}

\section{Random Initialisation}
\label{sec:ri}

Instead of going to higher-order networks, which come with a
substantial computational cost, random initialisation is another
simple idea for increasing the expressiveness of GNNs, actually
without a steep computational cost.

Recall that the initial feature map $\zeta^{(0)}$ of a GNN operating
on a graph $G$ is a $q^{(0)}$-dimensional vector whose first $\ell$
entries encode the vertex-labelling of $G$ and whose remaining
components are set to $0$. In the following, we always assume that
$q^{(0)}>\ell$. Instead of initialising the entries above $\ell$ to
$0$, we initialise them with random numbers, say, drawn uniformly from
the interval $[0,1]$. (For our theoretical considerations, the exact
distribution is not important. In practice, it has some effect, see
\cite{abbceygroluk21}). In the following, we speak
of GNNs with \emph{random node initialisation (RNI)}. To avoid
cumbersome terminology, let us assume that \emph{GNNs with RNI always admit
global readout.}

The computation
of a GNN with RNI is no longer deterministic but becomes a random
variable. Note that this random variable is isomorphism
invariant. To express a query or invariant, we must
quantify the error probability. We say that that a GNN with RNI
computing a vertex invariant $\xi$ (formally a random variable)
\emph{expresses} a unary query $Q$ if there are
$\epsilon,\delta<1/2$ such that
\begin{equation}
  \label{eq:5}
    \begin{cases}
    \Pr(\xi(G,v)\ge 1-\epsilon)\ge1-\delta&\text{if }v\in Q(G),\\
    \Pr(\xi(G,v)\le \epsilon)\ge1-\delta&\text{if }v\not\in Q(G).
  \end{cases}
\end{equation}
A similar definition can be made for Boolean queries.

To understand why GNNs with RNI can be more expressive than plain
GNNs, think of the query asking if a vertex is in a triangle. Since
$1$-WL cannot detect triangles (cf.~Example~\ref{exa:cr-inxp}),
neither can GNNs. Intuitively, the reason is that a GNN never detects
the origin of a message because vertices have no
identifiers. However, random initialisation with high probability
provides each node $v$ with a unique identifier (the entry in the
$(\ell+1)$st position of the initial state $\zeta^{(0)}(v)$). Thus the
GNN can detect a sequence of messages from a node $v$ to a
neighbour $w$ to a neighbour $x$ of $w$ and from there back to $v$.

Random node initialisation has often been used in practice as a
default for GNNs. Sato et al.~\cite{SatoYK21}
were the first to demonstrate the theoretical strength of GNNs with
RNI. It was shown by Abboud et
al.~\cite{abbceygroluk21} that GNNs with RNI have
the following universal approximation property. For a unary query $Q$
and a positive integer $n$, we say that a GNN with RNI
\emph{expresses $Q$ on graphs of order at most $n$} if conditions
\eqref{eq:5} are satisfied for all graphs $G$ of order at most $n$.

\begin{theorem}[\cite{abbceygroluk21}]\label{theo:rni}
  For unary query $Q$ and every $n\ge 1$ there is a
  GNN with RNI that expresses $Q$ on graphs of order at most $n$.
\end{theorem}

A similar theorem holds for Boolean queries, and even for graph
invariants and vertex invariants (see
\cite{abbceygroluk21}).

\begin{proof}[Proof sketch]
  We can use the random initialisation to create a vertex labelling
  that uniquely identifies each vertex, with high probability. With this labelling, the logic
  $\LCk2$ and hence GNNs can describe the graph up to isomorphism. We   use this to describe the property, essentially as a big disjunction
  over all pairs $(G,v)$ consisting of a graph $G$ order at most $n$
  and a
  vertex $v\in Q(G)$.
\end{proof}

While the theorem is non-uniform and the proof pays no attention to
computational efficiency in terms of the size $n$ of the input graphs,
various experiments
\cite{abbceygroluk21,SatoYK21}
have shown that random initialisation indeed increases the
expressiveness. Yet a more careful complexity-theoretic analysis
remains future work.

\section{Conclusions and Open Problems}
\label{sec:conc}

We have seen that the expressiveness of graph neural networks has
precise characterisations in terms of the Weisfeiler-Leman algorithm
and 2-variable counting logic. Understanding the expressiveness of
machine learning architectures is useful to guide us in the choice of
an appropriate architecture for a problem at hand and to compare
different architectures and approaches. Of course, expressiveness is only one aspect
of a machine learning algorithm, other important aspects like the
ability to generalise from the given data, and the computational
efficiency of learning and inference are not considered in this paper.

The tight correspondence between the expressiveness of GNNs and logical
expressiveness may open possibilities for neuro-symbolic integration,
that is, the integration of logic-based and statistical reasoning in
AI. The theorems presented describe which logical queries can be
expressed using GNNs. Of course, that does not mean that we can
actually learn GNN models representing these queries. With current
techniques, I would regard the question of
learnability mainly as an empirical question that can be
studied experimentally. In various contexts, it has been demonstrated
that GNNs for logical queries can be learned (for example,
\cite{abbceygroluk21,tonritwolgro20+}), but I
believe a more systematic empirical investigation might be worthwhile.

There are also many interesting theoretical questions that
remain open. Uniformity is an issue that comes up in several of the
result presented here. Theorems~\ref{theo:gnn-lower},
\ref{theo:gnn-wl-lower}, \ref{theo:kgnn}(2), and \ref{theo:rni} are
non-uniform expressiveness results: they state the existence of certain
GNNs that depend on the size of the input graph. By comparison,
Theorem~\ref{theo:gnn-logic} is uniform.

\begin{question}
  Is there a uniform version of Theorems~\ref{theo:gnn-lower}, that
  is, a recurrent GNN such that for all $t\ge0$ the vertex invariant $\xi^{(t)}$ computed by
  the $t$-th iteration of the GNN refines $\colref t$?

  The same question can be asked for Theorems~\ref{theo:gnn-lower} and \ref{theo:gnn-wl-lower}(2).
\end{question}

The colouring obtained by $1$-WL can be defined in 2-variable
fixed-point logic, and presumably the same holds for colour
refinement and a suitable modal fixed-point logic with counting. Thus, as a common
generalisation of the previous question and
Theorem~\ref{theo:gnn-logic}, we may ask the following.

\begin{question}
  Let $Q$ be a unary query expressible in a suitable modal
  (2-variable) fixed-point logic with counting. Is there a recurrent
  GNN (with global readout) expressing $Q$? 
\end{question}

So far, Theorem~\ref{theo:gnn-logic} has only been proved using the
linearised sigmoid function as activation function. 

\begin{question}
  For which activation functions does Theorem~\ref{theo:gnn-logic}
  hold? Is there a general condition (similar to the
  ``non-polynomial'' in Theorem~~\ref{theo:fnn-app})?
\end{question}


With Theorem~\ref{theo:logic-gnn}, we have a partial converse of
Theorem~\ref{theo:gnn-logic} tightening the connection between the
logic $\LGCk2$ and GNNs. It is an open question if a similar result
holds for $\LCk2$ and GNNs with global readout.

\begin{question}
  Is every unary query expressible by a GNN with global readout and also
  expressible in first-order logic expressible in the logic  $\LCk2$?
\end{question}

Of course, the uniformity question can also be asked for
Theorem~\ref{theo:rni}, but technically the situation is a bit
different, and it seems very unlikely that we obtain a uniform version
of that theorem. It should not be too hard to prove this.

\begin{question}
  Is there a Boolean query not expressible by a
  (possibly recurrent) GNN with random node initialisation. 
\end{question}

But still, there might be interesting uniform expressiveness results for
GNNs with random node initialisation. Expressiveness of queries by GNNs
with random node initialisation is related to logical
expressiveness by order-invariant formulas (see, for example,
\cite[Chapter 5]{lib04}).

\begin{question}
  Can we express all Boolean or unary queries expressible by an
  order-invariant $\LC^2$-formula by a GNN with random node
  initialisation? 
\end{question}

\begin{question}
  Can we express all Boolean or unary queries expressible by an
  order-invariant 2-variable fixed-point formula with counting by a
  recurrent GNN with random node
  initialisation? 
\end{question}

Of course there is no need to only consider logical queries.

\begin{question}
  Can we express all Boolean or unary queries computable in polynomial
  time by a recurrent GNN with random node initialisation?
\end{question}

The previous question also has an interesting non-uniform version.

\begin{question}
  Let $Q$ be a Boolean query computable by a (non-uniform) family of
  Boolean threshold circuits of polynomial size. Is there a family
  $(N_n)_{n\ge 1}$ of polynomial size GNNs with random node
  initialisation such that $N_n$ expresses $Q$ on input graphs of size
  $n$?
\end{question}

\printbibliography

\newpage
\appendix
\section*{Proof of Theorem~\ref{theo:owl}}

The following lemma contains the essence of the proof.

\begin{lemma}\label{lem:owl}
  let $k\ge 1$. Then for all graphs $G,G'$, all
  $\vec v\in V(G)^{k+1},\vec v'\in V(G')^{k+1}$, and all $t\in\Nat$,
  the following are equivalent:
  \begin{eroman}
  \item $\owl{k+1}t(G,\vec v)=\owl{k+1}t(G',\vec v')$;
  \item $\atp{k+1}(G,\vec v)=\atp{k+1}(G',\vec v')$ and $\wl
    kt\big(G,\vec v[/i]\big)=\wl
    kt\big(G',\vec v'[/i]\big)$ for all $i\in[k+1]$.
  \end{eroman}
\end{lemma}

\begin{proof}
  We fix the graphs $G,G'$.
  The proof is by induction on $t$.
 The base step $t=0$ is trivial. For the inductive step $t\to t+1$, let $\vec v\in V(G)^{k+1},\vec v'\in V(G')^{k+1}$.

  To prove the implication (i)$\implies$(ii) we assume that
  $\owl{k+1}{t+1}(G,\vec v)=\owl{k+1}{t+1}(G',\vec v')$. Then
  $\atp{k+1}(G,\vec v)=\atp{k+1}(G',\vec v')$, because
  $\owl{k+1}{t+1}$ refines $\atp{k+1}$.

  Let $i\in[k+1]$. We need to prove that
  $\wl k{t+1}(G,\vec v[/i])=\wl k{t+1}(G',\vec v'[/i])$. By the
  definition of $\owl{k+1}{t+1}$ we have
  \begin{equation}\label{eq:1}
    \begin{array}{l}
      \Biglmulti \owl kt\big(G,\vec
      v[w/i]\big)\Bigmid w\in V(G)\Bigrmulti\\
      \hspace{1.2cm}=\Biglmulti \owl kt\big(G',\vec
      v'[w'/i]\big)\Bigmid w'\in V(G')\Bigrmulti.
    \end{array}
  \end{equation}
  Thus there is a bijection $h:V(G)\to V(G')$ such that
  $\owl kt\big(G,\vec v[w/i]\big)=\owl kt\big(G',\vec v'[h(w)/i]\big)$
  for all $w\in V(G)$. By the induction hypothesis, this implies
  $\atp{k+1}\big(G,\vec v[w/i]\big)=\atp{k+1}\big(G',\vec
  v'[h(w)/i]\big)$ and
  $\wl kt \big(G,\vec v[w/i][/j]\big)=\wl kt \big(G',\vec
  v'[h(w)/i][/j]\big)$ for all $j\in[k+1]$. For $j=i$, this implies
  $\wl kt(G,\vec v[/i])=\wl kt(G',\vec v'[/i])$.  Moreover, for all
  $j\in[k]$, it implies
  $\wl kt \big(G,\vec v[/i][w/j]\big)=\wl kt \big(G',\vec
  v'[/i][h(w)/j]\big)$. Thus by the definition of $\wl k{t+1}$, it
  follows that
  $\wl k{t+1} \big(G,\vec v[/i]\big)=\wl k{t+1} \big(G',\vec v'[/i]\big)$.

  It remains to prove (ii)$\implies$(i). Suppose that
  $\atp{k+1}(G,\vec v)=\atp{k+1}(G',\vec v')$ and
  $\wl k{t+1}(G,\vec v[/i])=\wl k{t+1}(G',\vec v'[/i])$ for all
  $i\in[k+1]$. As $\wl k{t+1}$ refines $\wl kt$, by the inductive
  hypothesis this implies $\owl{k+1}t(G,\vec v)=\owl{k+1}t(G',\vec
  v')$. Thus by the definition of $\owl{k+1}{t+1}$, to prove that $\owl{k+1}{t+1}(G,\vec v)=\owl{k+1}{t+1}(G',\vec
  v')$ we need to prove that for all $i\in[k+1]$ we have \eqref{eq:1}.

  So let $i\in[k+1]$. Since $\wl k{t+1}(G,\vec v[/i])=\wl
  k{t+1}(G',\vec v'[/i])$, there is a bijection $h:V(G)\to V(G')$ such
  that for all $w\in V(G)$ we have $\atp{k+1}(G,\vec
  v[/i]w)=\atp{k+1}(G',\vec v'[/i]h(w))$ and $\wl kt(G,\vec
  v[/i][w/j])=\wl kt(G',\vec v'[/i][h(w)/j])$ for all $j\in[k]$. This implies $\atp{k+1}(G,\vec
  v[w/i])=\atp{k+1}(G',\vec v'[h(w)/i])$ and $\wl kt(G,\vec
  v[w/i][/j])=\wl kt(G',\vec v'[h(w)/i][/j])$ for all
  $j\in[k+1]$ (for $j=i$ we use $\wl k{t+1}(G,\vec v[/i])=\wl
  k{t+1}(G',\vec v'[/i])$). Thus by the induction hypophypothesisthesis, 
  $\owl {k+1}t(G,\vec v[w/i])=\owl {k+1}t(G,\vec v[h(w)/i])$ for all
  $w\in V(G)$. As $h$ is a bijection, \eqref{eq:1} follows.
\end{proof}

\begin{proof}[Proof of Theorem~\ref{theo:owl}]
  Let $G,G'$ be a graphs and $t\in\Nat$.
  
  To prove assertion (1), suppose that $G,G'$ are not distinguished by
  $\owl{k+1}t$. Then there is a bijection $f:V(G)^{k+1}\to
  V(G')^{k+1}$ such that $\owl{k+1}t(G,\vec v)=\owl{k+1}t(G',f(\vec
  v))$ for all $\vec v\in V(G)^{k+1}$.

  We define a bijection
  $g:V(G)^k\to V(G')^k$ as follows: for $\vec v=(v_1,\ldots,v_k)$, let
  $\vec v_+\coloneqq(v_1,\ldots,v_k,v_k)$ and 
  $\vec v'_+=(v_1',\ldots,v_{k+1}')\coloneqq
  f(\vec v_+)$. Then $v'_{k}=v'_{k+1}$, because
  $\owl{k+1}t(G,\vec v_+)=\owl{k+1}t(G',
  \vec v_+')$ and thus $\atp{k+1}(G,\vec v_+)=\atp{k+1}(G',
  \vec v_+')$ by Lemma~\ref{lem:owl}. We let $g(\vec v)\coloneqq
  (v_1',\ldots,v_{k}')=:\vec v'$. Then $g$ is indeed a bijection from
  $V(G)^k$ to $V(G')^k$, and since $\vec v=\vec v_+[/k+1]$ and $g(\vec
  v)=\vec
  v'=\vec v_+'[/k+1]$, we have $\wl kt(G,\vec v)=\wl kt(G',\vec v')$
  by Lemma~\ref{lem:owl}. Thus $g$ is a bijection from
  $V(G)^k$ to $V(G')^k$ that preserves $\wl kt$, and this implies that
  $\wl kt$ does not distinguish $G$ and $G'$.

  To prove (2), assume that $G$ and $G'$ are not distinguished by
  $\wl{k}{t+1}$. Then there is a bijection $g:V(G)^k\to V(G')^k$ such
  that $\wl{k}{t+1}(G,\vec v)=\wl{k}{t+1}(G',f(\vec v))$ for all
  $\vec v\in V(G)^{k}$.

  Let $\vec v\in V(G)^k$ and $\vec v'\coloneqq f(\vec v)$. By the
  definition of $\wl k{t+1}$, we have
  $\wl kt(G,\vec v)=\wl kt(G',\vec v')$ and there is a bijection
  $h_{\vec v}:V(G)\to V(G')$ such that for all $w\in V(G)$ we have
  $\atp{k+1}(G,\vec v w)=\atp{k+1}(G',\vec v'h_{\vec v}(w))$ and
  $\wl kt(G,\vec v[w/i])=\wl kt(G',\vec v'[h_{\vec v}(w)/i])$ for all
  $i\in[k]$. This implies
  $\wl kt(G,\vec vw[/i])=\wl kt(G',\vec v'h_{\vec v}(w)[/i])$ for all
  $i\in[k]$. Moreover, since $\vec v=\vec vw[/k+1]$ we also have
  $\wl kt(G,\vec vw[/k+1])=\wl kt(G',\vec v'h_{\vec
    v}(w)[/k+1])$. Thus by Lemma~\ref{lem:owl}, it follows that
  $\owl{k+1}t(G,\vec vw)=\owl{k+1}t(G',\vec v'h_{\vec v}(w))$.

  We
  define a bijection $f:V(G)^{k+1}\to V(G)^{k+1}$ by $f(\vec
  vw)\coloneqq g(\vec v)h_{\vec v}(w)$ for all $\vec v\in V(G)^k,w\in
  V(G)$. Then $f$ preserves $\owl{k+1}t$, and thus $\owl{k+1}t$ does
  not distinguish $G$ and $G'$.
\end{proof}

\section*{Proof of Theorem~\ref{theo:cr-logic}}

  The proof is by induction on $t$. For the base case $t=0$, note that
  a $\LGCkq20$ formula $\phi(x)$ is a Boolean combination of
  atomic ``Label'' formulas $P_i(x)$; formulas $E(x,x)$ (always false)
  and $x=x$ (always true) are not needed. The colouring
  $\colref 0=\col$ captures precisely the label information.

  For the inductive step $t\ge t+1$, we first prove the implication
  (i)$\implies$(ii). Assume that $\colref
  {t+1}(G,v)=\colref {t+1}(G',v')$. Let $\phi(x)\in\LGCkq2{t+1}$. Then
  $\phi$ is a Boolean combination of atomic formulas $P_i(x)$ and
  formulas $\exists^{\ge p}
  y\big(E(x,y)\wedge\psi(y)\big)$, where $\psi(y)\in\LGCkq2t$. We shall prove that $G$ and $G'$
  satisfy the same formulas of these types. As $\colref{r+1}$
  refines $\col$ and $\colref
  {t+1}(G,v)=\colref {t+1}(G',v')$, we have $\col(G,v)=\col(G',v')$
  and thus $G\models P_i(v)\iff G'\models P_i(v')$ for all
  $i$. Consider a formula $\phi'(x)\coloneqq\exists^{\ge p}
  y\big(E(x,y)\wedge\psi(y)\big)$, where $\psi(y)\in\LGCkq2t$. By the
  induction hypothesis, for every colour $c$ in the range of $\colref t$,
  either all vertices of colour $c$ or
  none of the vertices of colour $c$ satisfy the formula
  $\psi(y)$. Let $c_1,\ldots,c_q$ be the colours in the range of $\colref
  t$ such that all vertices  of colour $c_j$ satisfy 
  $\psi(y)$. Since $\colref
  {t+1}(G,v)=\colref {t+1}(G',v')$, for each $j\in[q]$ the number of
  vertices $w\in V(G)$ such that $vw\in E(G)$ and $\colref t(G,w)=c_j$
  equals the number of
  vertices $w'\in V(G')$ such that $v'w'\in E(G')$ and $\colref
  t(G',w')=c_j$. Thus the the number of
  vertices $w\in V(G)$ such that $vw\in E(G)$ and $G\models\psi(w)$
  equals the number of
  vertices $w'\in V(G')$ such that $v'w'\in E(G')$ and
  $G'\models\psi(v')$. It follows that $G\models\phi'(v)\iff
  G'\models\phi'(v')$.

  To prove the converse implication (ii)$\implies$(i), we assume that
  for all formulas $\phi(x)\in\LGCkq2{t+1}$ it holds that
  $G\models\phi(v)\iff G'\models\phi(v')$. Let $c_1,\ldots,c_q$ be the (finite) list of
  colours in the range of $\colref t(G)\cup\colref t(G')$. By the induction
  hypothesis, for all distinct $i,j\in[q]$ there is a formula
  $\psi_{ij}(x)\in\LGCkq2t$ that is satisfied by all vertices of
  colour $c_i$, but by no vertices of colour $c_j$. Thus the formula
  $\phi_i(x)\coloneqq\bigwedge_{j\neq i} \psi_{ij}(x)$ is satisfied exactly by
  the vertices of colour $c_i$ and by no other vertices in $V(G)\cup
  V(G')$.

  Suppose for contradiction that there is a colour $c_i$ such that
  \begin{align*}
    p_i&\coloneqq\big|\big\{ w\in N^G(v)\bigmid \colref
         t(G,w)=c_i\big\}\big|\\
   &\hspace{1cm} \neq \big|\big\{ w'\in
    N^{G'}(v')\bigmid \colref t(G',w')=c_i\big\}\big|=:p_i'.
  \end{align*}
  Without loss of generality we assume that $p_i>p_i'$. Let
  $\phi(x)\coloneqq\exists^{\ge
    p_i}y\big(E(x,y)\wedge\phi_i(y)\big)$. Then
  $\phi(x)\in\LGCkq2{t+1}$ and $G\models\phi(v)$, but
  $G'\not\models\phi(v')$. This is a contradiction.

  Thus $v$ and $v'$ have the same numbers of neighbours of each colour
  in the range of $\colref t$. By definition, this means that
  $\colref{t+1}(G,v)=\colref{t+1}(G',v')$. 
\qed

\end{document}